\documentclass[final,hidelinks,onefignum,onetabnum]{siamart220329}

\usepackage{lipsum}
\usepackage{float}
\usepackage{amsmath, amsfonts}
\usepackage{graphicx}
\usepackage{subcaption}
\usepackage[section]{placeins}
\usepackage{epstopdf}
\usepackage{algorithmic}
\usepackage[export]{adjustbox}
\usepackage[normalem]{ulem}
\ifpdf
  \DeclareGraphicsExtensions{.eps,.pdf,.png,.jpg}
\else
  \DeclareGraphicsExtensions{.eps}
\fi
\usepackage[textsize=tiny]{todonotes}
\usepackage{mathtools}

\newsiamremark{remark}{Remark}
\newsiamremark{hypothesis}{Hypothesis}
\crefname{hypothesis}{Hypothesis}{Hypotheses}
\newsiamthm{claim}{Claim}

\headers{Coupled Input-Output Dimension Reduction}{Q. Chen, E. Arnaud, R. Baptista, and O. Zahm}

\title{Coupled Input-Output Dimension Reduction:\\Application to goal-oriented Bayesian experimental design and global sensitivity analysis}

\author{
Qiao Chen\thanks{Université Grenoble Alpes, Inria Grenoble, France (\email{elise.arnaud@univ-grenoble-alpes.fr}, \email{qiao.chen@inria.fr}, \email{olivier.zahm@inria.fr}).}
\and Élise Arnaud\footnotemark[1]
\and Ricardo Baptista\thanks{Computing and Mathematical Sciences, California Institute of Technology, Pasadena, CA (\email{rsb@caltech.edu}).}
\and Olivier Zahm\footnotemark[1]
}

\usepackage{amsopn}

\DeclareMathOperator{\Tr}{Tr}

\DeclareMathOperator{\supp}{supp}
\DeclareMathOperator{\Ker}{Ker}
\DeclareMathOperator{\Imag}{Im}
\DeclareMathOperator{\Cov}{Cov}
\DeclareMathOperator{\Var}{Var}

\DeclareMathOperator{\KL}{\mathbb{D}_{KL}}

\newcommand{\tot}{\textup{tot}}
\newcommand{\cl}{\textup{cl}}
\newcommand{\PCA}{\textup{PCA}}
\newcommand{\Coupled}{\textup{Coupled}}

\renewcommand*{\d}{\mathop{}\!\mathrm{d}}

\providecommand{\bbE}{\mathbb{E}}

\providecommand{\bbN}{\mathbb{N}}

\providecommand{\bbR}{\mathbb{R}}

\providecommand{\CC}{\mathcal{C}}

\providecommand{\CI}{\mathcal{I}}

\providecommand{\CN}{\mathcal{N}}

\newcommand{\VG}{{\mathbf{G}}}

\ifpdf
\hypersetup{
  pdftitle={Coupled Input-Output Dimension Reduction},
  pdfauthor={Q. Chen, E. Arnaud, R. Baptista, and O. Zahm}
}
\fi

\begin{document}

\maketitle

\begin{abstract}
We introduce a new method to jointly reduce the dimension of the input and output space of a function between high-dimensional spaces.
Choosing a reduced input subspace influences which output subspace is relevant and vice versa.
Conventional methods focus on reducing either the input or output space, even though both are often reduced simultaneously in practice.
Our coupled approach naturally supports goal-oriented dimension reduction, where either an input or output quantity of interest is prescribed.  
We consider, in particular, goal-oriented sensor placement and goal-oriented sensitivity analysis, which can be viewed as dimension reduction where the most important output or, respectively, input \emph{components} are chosen. 
Both applications present difficult combinatorial optimization problems with expensive objectives such as the expected information gain and Sobol' indices.
By optimizing gradient-based bounds, we can determine the most informative sensors and most influential parameters as the largest diagonal entries of some diagnostic matrices, thus bypassing the combinatorial optimization and objective evaluation.
 
\end{abstract}



\section{Introduction}
High dimensionality poses a challenge in many fields of applied mathematics, seriously limiting the effectiveness of established solution methods. 
Tasks like approximation, sampling, and optimization become exponentially harder as the size of the respective exploration spaces explodes with growing dimensions.
And yet, the dimensionality of computational problems continues to rise due to improved processing capacities that allow for higher resolution simulations and an unprecedented accumulation of data. 
In response to this challenge, dimension reduction methods seek to leverage the underlying low-dimensional structures present in many problems.

Classical dimension reduction methods target either the input or the output space of a model. For example, truncated Karhunen-Loève expansion \cite{loeve1977KL} is often applied to reduce input parameters, while the reduced basis method \cite{almroth1978RB}, proper orthogonal decomposition \cite{sirovich1987turbulence,lumley1967structure} and principal component analysis (PCA) \cite{hotelling1933PCA} are used on the output states.
Derivative-based reduction methods like ours have been mostly proposed for the parameter space like the active subspace \cite{constantine2014kriging,zahm2020gradient} and likelihood informed subspace \cite{cui2014LIS,zahm2022certified} methods.
Although input and output reduction methods are often applied simultaneously, the two spaces are commonly treated separately (see e.g. \cite{alexanderian2019InputOutputKL,bhattacharya2021model,o2024derivative}). However, choosing an input subspace inevitably impacts the relevant output subspace and vice versa. Taking this interplay into account can significantly reduce computational cost by allowing lower-dimensional approximations with the same level of accuracy \cite{alexanderian2019InputOutputKL}.

Few works consider the coupling between the input and output space reduction, though the coupling is often only algorithmic and one-directional.
For example, \cite{lieberman2010ParameterStateRB} uses a greedy reduced basis method for output state reduction and re-purposes the snapshot parameters to span a reduced parameter space.
Conversely, \cite{cui2016ParameterStateLIS} employs likelihood informed parameter subspaces, and then reduces the output using PCA on snapshots sampled from a reduced parameter distribution.
The operator learning framework of \cite{lu2021learning} optimizes the output reduction (seen here as a linear decoder parametrized by neural networks) after fixing a reduced input, e.g., using Karhunen-Loève or active subspaces; see also the discussion in \cite{kovachki2024operator}.
\cite{baptista2022JointDR} derive error bounds for joint input-output dimension reduction which permits one to balance the contribution of the two reductions to meet an overall level of accuracy.

The coupling between input and output space is also the motivation behind goal-oriented methods. These approaches tailor the dimension reduction of the input or output space to specific quantities of interest in the opposite space to enhance computational efficiency and significance of the results. 
For example, optimal sensor placement can be viewed as dimension reduction that selects the most informative output \emph{components}. The goal-oriented optimal sensor placement problem \cite{feng2019goalOED,alexanderian2018GoalOED,chen2021GoalEIG} seeks optimality with respect to a prescribed lower dimensional parameter of interest (instead of the whole parameter). On the other hand, sensitivity analysis can be viewed as input dimension reduction where parameter components that cause the most (or least) output variations are identified. Goal-oriented sensitivity analysis considers sensitivities with respect to a specified output of interest.

\subsection{Contribution and Outline}

We develop a coupled dimension reduction method that accounts for the interdependence between the input and output space of a non-linear function $G\colon\bbR^d\rightarrow \bbR^m$.
As a consequence of the coupling, our method naturally supports goal-oriented dimension reduction. 
Using the Poincar\'{e} and a Cram\'{e}r-Rao-like inequality, we establish gradient-based upper and lower bounds for several joint and goal-oriented dimension reduction objectives. These bounds provide us with computable error estimates and an easy optimization algorithm based on the eigendecomposition of two diagnostic matrices
\begin{align*}
	H_X(V_s) &= \bbE\left[ \nabla G(X)^\top V_s  V_s^\top \nabla G(X) \right], \\
	H_Y(U_r) &= \bbE\left[ \nabla G(X) U_r  U_r^\top \nabla G(X)^\top \right],
\end{align*}
for the input and output space, respectively. Here, $\nabla G(X)\in\bbR^{m\times d}$ denotes the Jacobian of $G$, while $U_r\in\bbR^{d\times r}$ and $V_s\in\bbR^{m\times s}$ are orthogonal matrices spanning the reduced input and output space, respectively. We note that each diagnostic matrix is a function of a specified subspace in the opposing space. In the goal-oriented case where a fixed $U_r$ prescribes a parameter of interest $U_r^\top X$, the optimal $V_s^*(U_r)$ consists of the dominant eigenvectors of $H_Y(U_r)$. Equivalently, for a given output of interest $V_s^\top G(X)$, the optimal $U_r^*(V_s)$ consists of the dominant eigenvectors of $H_X(V_s)$. To obtain coupled subspaces $(U_r^*, V_s^*)$, we propose an alternating eigendecomposition of both diagnostic matrices.

In the goal-oriented optimal sensor placement scenario, we seek $V_{\tau}(U_r) = [e_{\tau}^m]$ consisting of canonical basis vectors $e_i^m\in\bbR^m$ indexed by $\tau\subseteq\{1,\ldots,m\}$. By maximizing a derivative-based lower bound of the expected information gain, we obtain the optimal $\tau^*$ as the index set of the largest diagonal entries of $H_Y(U_r)$. Similarly, in the goal-oriented sensitivity analysis scenario, we seek $U_{\tau}(V_r) = [e_{\tau}^d]$ consisting of canonical basis vectors $e_i^d\in\bbR^d$. The solution $\tau^*\subseteq\{1,\ldots,d\}$ that optimizes a derivative-based bound of the total Sobol' index is given as the largest diagonal entries of $H_X(V_s)$.
In both applications, we circumvent the expensive evaluation of the objective functions and the combinatorial optimization problem by maximizing our derivative-based bounds.

The remainder of this article is organized as follows. We begin in \Cref{sec:Objectives} by introducing several coupled and goal-oriented dimension reduction objectives. \Cref{sec:Coupled} contains our main results. We derive derivative-based upper and lower bounds for the input-output dimension-reduction objectives, as well as an algorithm that computes the optimal coupled subspaces with respect to the bounds. In \Cref{sec:Goal}, we consider two goal-oriented applications. In particular, we look at goal-oriented Bayesian optimal experimental design in \Cref{subsec:GoalBOED} and goal-oriented global sensitivity analysis in \Cref{subsec:GoalGSA}.
We finally demonstrate our theoretical results on numerical examples in \Cref{sec:Numerics}.

\section{Problem Setting and Objectives}\label{sec:Objectives}

Let $G\colon\bbR^d\rightarrow\bbR^m$ be a deterministic non-linear function of a random input parameter vector $X$. We reduce the input vector $X$ and output $Y=G(X)$ by retaining only some low-dimensional projections
\begin{align*}
	X_r &= U_r^\top X,\\
	Y_s &= V_s^\top Y ,
\end{align*}
where $U_r\in\bbR^{d\times r}$, $V_s\in\bbR^{m\times s}$ are two matrices with $r\ll d$ and $s\ll m$ orthonormal columns.
The identification of $U_r$ and $V_s$ depends on the objective at hand. 
In the following, we introduce four potential objectives. The first two consider general function approximation and Bayesian inference problems, while the second two consider goal-oriented applications of our framework.

\paragraph{Reduced order modelling}
In reduced order modelling (ROM), $G$ is typically a parametrized partial differential equation (PDE) model taking parameter $X$ and returning the (discretized) solution $G(X)$ of the PDE.
ROM aims to construct a fast-to-evaluate approximation $\widetilde G$  that can be used in place of $G$ for real-time or multiple-query problems such as control, optimization or sampling. A possible formalization of the problem is to find matrices $U_r$ and $V_s$ which minimize the $L^2$-error
\begin{align}
	\min_{U_r, V_s}\; &\bbE\bigl[\|G(X)-\widetilde G(X)\|^2\bigr], \tag{O1}\label{eq:objective_L2}\\
	\text{where}\quad &\widetilde G(x) = V_s\, \widetilde g(U_r^\top x) +  \widetilde G_\perp ,\label{eq:def_G_tilde}
\end{align}
for some optimal low-dimensional function $\widetilde g\colon\bbR^r\rightarrow\bbR^s$ and vector $\widetilde G_\perp\in\bbR^m$ to be specified later in \Cref{thm:optimalG}.
If the minimum value of the objective \eqref{eq:objective_L2} is close to 0, then $G\approx\widetilde G$ is essentially constant along the $(d-r)$-dimensional subspace $\Ker(U_r^\top)$ and the variation in its output is confined in the $m$-dimensional affine subspace $\Imag(V_s) +  \widetilde G_\perp$.

\paragraph{Bayesian inference}
In Bayesian inference, the aim is to compute statistics for the posterior distribution $\pi_{X|Y}$.
Existing knowledge on $X$, represented by the prior distribution $\pi_X$, is thereby updated with new information from a noisy observation $Y = G(X) + \eta$, 
where $\eta\sim\mathcal{N}(0,\sigma^2 I_m)$ denotes some additive Gaussian noise with $I_m\in\bbR^{m\times m}$ being the identity matrix. While one may consider more general likelihood functions, we will restrict ourselves to the setting of additive and Gaussian observational noise in this work.
By this definition, the likelihood $\pi_{Y|X}$ is Gaussian so that Bayes theorem yields the posterior density, which is known up to a normalization constant, as
\begin{equation*}
 \pi_{X|Y}(x|y) \propto \exp\left(-\tfrac{1}{2\sigma^2} \left\|y - G(x)  \right\|^2 \right) \pi_X(x).
\end{equation*}
In the case of a high-dimensional $X$, standard Markov-chain Monte Carlo (MCMC) methods need an excessive amount of iterations (and thus of model evaluations) to accurately sample the posterior.
Reducing the dimension of $X$ aims at identifying the components $X_r$ which are most \emph{informed} by the data $Y$ in order to run the inference algorithm, e.g., MCMC on a low-dimensional subspace \cite{cui2014LIS}.
Reducing the dimension of $Y$ permits one to only consider the most \emph{informative} data $Y_s$, typically yielding smaller autocorrelation between posterior samples and further reducing the number of required MCMC iterations \cite{baptista2022JointDR}.
We formalize this problem as finding $U_r$ and $V_s$ that minimize the posterior error with respect to the expected Kullback-Leibler (KL) divergence
\begin{align}
	\min_{U_r, V_s}\; &\bbE_{Y}\left[\KL\left(\pi_{X|Y}\|\,\widetilde\pi_{X|Y}\right)\right], \tag{O2}\label{eq:objective_posterior}\\
	\text{where}\quad &\widetilde \pi_{X|Y}(x|y) \propto \pi_{Y_s|X_r}(y_s|x_r)\pi_{X}(x)  ,\label{eq:def_pi_tilde}
\end{align}
with $\KL(\nu\|\mu) = \int\log(\d\nu/\d\mu)\d\nu$.
For \eqref{eq:objective_posterior} to be close to $0$, the parameter $X$ needs to be essentially independent of the observations in the subspace $\Ker(V_s)$, and observations of $Y$ should primarily provide information within the parameter subspace $\Imag(U_r)$.  
To sample from the approximate density $\widetilde\pi_{X|Y}$, we can independently sample $X_r\in\bbR^r$, which depends on the projected observation (see algorithms to sample the reduced posterior in~\cite[Section 6]{baptista2022JointDR}), and its orthogonal complement $X_\perp\in\bbR^{d-r}$, which only depends on the prior.

\paragraph{Goal-oriented Bayesian optimal experimental design}
Bayesian optimal experimental design (BOED) aims to select a set of sensors $\tau \subseteq \{1, \ldots, m\}$ with $|\tau|=s$ that maximizes the information gained about the parameter that is to be inferred. Measurements from the sensors $\tau$ are gathered in the vector $Y_\tau= V_\tau^\top Y$, where $V_\tau = [e_\tau^m]\in \bbR^{m\times s}$ contains the canonical basis vectors $e_{i}^m\in\bbR^m$ indexed by $\tau$. The information content of the experimental design $V_\tau$ can be quantified by the expected information gain (EIG), which is defined as the KL-divergence from the prior to the posterior in expectation over the data $\Phi(V_\tau)=\bbE_{Y}[\KL(\pi_{X|Y_\tau}\|\,\pi_{X})]$. When the inference problem focuses on a specific linear quantity of interest in the parameter $X_r = U_r^\top X$, the sensor placement can be improved by maximizing the \emph{goal-oriented EIG}
\begin{equation}\label{eq:objective_goalEIG}
	\max_{V_\tau}\;\Phi(V_\tau|U_r) = \max_{V_\tau}\;\bbE_{Y}\left[\KL\left(\pi_{X_r|Y_\tau}\|\,\pi_{X_r}\right) \right] , \tag{O3}
\end{equation}
where $\pi_{X_r|Y_\tau}$ and $\pi_{X_r}$ are the corresponding marginalized densities. BOED is notoriously challenging as each evaluation of the EIG requires a costly nested Monte-Carlo estimator \cite{ryan2003NMC_EIG}. Existing methods often resort to linear Gaussian models for which the EIG can be computed explicitly \cite{alexanderian2016fast}. 
Nevertheless, it remains a combinatorial optimization problem that typically only allows for sub-optimal greedy solutions \cite{krause2008GreedySensorsGP,cohen2018GreedySensorsPBDW}.

\paragraph{Goal-oriented global sensitivity analysis}
Assuming independent inputs $X_i$, $i=1,\ldots,d$, global sensitivity analysis (GSA) aims to identify which set $\tau\subseteq\{1,\ldots,d\}$ of $r=|\tau|$ input factors causes the least output variations in order to fix them to a nominal value and simplify the model. A widely used sensitivity index is the total Sobol' index $S_\tau^{\tot} = 1-\frac{\Tr(\Cov( G(X)| X_{-\tau} ))}{\Tr(\Cov( G(X)  ))}$ with $X_{-\tau} = {U_{-\tau}}^\top X\in\bbR^{d-r}$ where $U_{-\tau} = [e_{-\tau}^d]\in\bbR^{d\times (d-r)}$ contains the canonical basis vectors $e_{i}^d\in\bbR^d$ indexed by the complement of $\tau$ \cite{gamboa2014multidim}. A smaller $S_\tau^{\tot}$ implies that the output is less sensitive to the parameters indexed by $\tau$. Thus, the \textit{optimal} set of input factors $\tau$ for factor fixing minimize the total Sobol' index $S_\tau^{\tot}$.
When there are particular output directions of interest $V_s^\top G(X)$, it is more suitable to consider minimizing the \emph{goal-oriented total Sobol' index}
\begin{equation}\label{eq:objective_goalGSA}
		\min_{U_\tau}\;S^{\tot}(U_\tau|V_s) 
		= \min_{U_\tau}\; 1-\frac{\Tr\left(\Cov\left(\left. V_s^\top G(X)\right| U_{-\tau}^\top X \right)\right)}{\Tr\left(\Cov\left( V_s^\top G(X)  \right)\right)}, \tag{O4}
\end{equation}
where $U_{\tau} = [e_\tau^d]\in\bbR^{d\times r}$.
GSA faces similar challenges as BOED. First, the objective function is expensive to evaluate and the development of fast computational methods is an ongoing research topic \cite{gamboa2022rank_statistics,qian2018MFMC}. Secondly, identifying the optimal set that minimizes the total Sobol' index poses a combinatorial problem. In practice, only single-element sets $\tau=\{ i\}$ are considered and Sobol' indices are computed for all $i=1,\ldots,d$ to rank them by importance.

\section{Coupled Input-Output Dimension Reduction}\label{sec:Coupled}

The following Lemmas allow us to address both the posterior approximation problem \eqref{eq:objective_posterior} as well as the goal-oriented BOED problem \eqref{eq:objective_goalEIG} by solving the $L^2$-approximation problem of the forward model \eqref{eq:objective_L2}. The proofs can be found in \Cref{app:L2vsPosterior,app:L2vsEIG}.

\begin{lemma}\label{Lem:L2vsPosterior}
	With the above notations, for any $U_r,V_s$ and $\widetilde G$ as in \eqref{eq:def_G_tilde}, we have
	\begin{align}\label{eq:KLposterior}
		\bbE_{Y}\left[\KL\left(\pi_{X|Y}\|\,\widetilde\pi_{X|Y}\right) \right]
		\leq \tfrac{1}{2\sigma^2} \bbE\big[ \| G(X) - \widetilde G(X) \|^2 \big].
	\end{align}
\end{lemma}
\begin{lemma}\label{Lem:L2vsEIG}
	For any $U_r,V_s$ and $\widetilde G$ as in \eqref{eq:def_G_tilde}, we have
	\begin{equation}\label{eq:GoalEIG_decomp}
		\Phi(V_s| U_r) \geq
		\bbE_{Y}\left[\KL\left(\pi_{X|Y}\|\,\pi_{X}\right) \right]
		- \tfrac{1}{2\sigma^2}\bbE\big[\|G(X)-\widetilde G(X)\|^2\big].
	\end{equation}
\end{lemma}

Inequality \eqref{eq:KLposterior} shows that minimizing the $L^2$-error of $\widetilde G$ corresponds to minimizing an upper bound of the posterior error of $\widetilde\pi_{X|Y}$.
At the same time, minimizing the $L^2$-error also maximizes a lower bound of the goal-oriented EIG $\Phi(V_s|U_r)$ according to \eqref{eq:GoalEIG_decomp}. For this reason, we will focus on the $L^2$-approximation of $\widetilde G$ in this section. The link to the goal-oriented GSA problem \eqref{eq:objective_goalGSA} will be addressed later in \Cref{prop:GradvsGSA}.
For any fixed pair of matrices $(U_r,V_s)$, the following proposition gives analytical expressions for the optimal function $\widetilde G$ in \eqref{eq:def_G_tilde}.
The proof is given in \Cref{app:ProofOptimalG}.

\begin{proposition}\label{thm:optimalG}
	Let $X$ be a random vector taking values in $\bbR^d$ and let $G\colon\bbR^d\rightarrow\bbR^m$ be such that $\bbE[\|G(X)\|^2]<\infty$.
	For any matrices $U_r\in\bbR^{d\times r}$ and $V_s\in\bbR^{m\times s}$ with orthogonal columns, we define $G^*\colon\bbR^d\rightarrow\bbR^m$ as
	\begin{equation*}
		G^*(x) = V_s  g^*(U_r^\top x) +  G_\perp^* ,
	\end{equation*}
	where $ g^*(x_r) =  V_s^\top \bbE[ G(X) |\, U_r^\top X = x_r ] $ and $ G_\perp^* =  (I_m-V_s V_s^\top)\,\bbE[G(X)]$.
	Then $G^*$ minimizes $\widetilde G\mapsto \bbE[\|G(X) - \widetilde G(X)\|^2]$ over the set of square integrable functions in the form of \eqref{eq:def_G_tilde}.
	Furthermore, it holds that
	\begin{align}
		\bbE\left[\|G(X)- G^*(X)\|^2\right] &= \Tr\left(\Cov\left(G(X)\right)\right) - \Tr\left(V_s^\top \Cov\left(\bbE[G(X)| U_r^\top X ]\right) V_s\right) , \label{eq:error_decomp} 
	\end{align}
	where $\Cov(\bbE[ G(X) |\, U_r^\top X ]) \in\bbR^{m\times m}$ denotes the covariance of $\bbE[ G(X) |\, U_r^\top X ]$.
\end{proposition}

Proposition \ref{thm:optimalG} allows us to obtain optimal $U_r$ and $V_s$ by minimizing the $L^2$-error of the forward model as given by \eqref{eq:error_decomp}, the minimum value for the objective in~\eqref{eq:objective_L2}. After that, one can approximate the low-dimensional function $g^*$ using state-of-the-art approximation methods (see e.g. \cite{adcock2022sparse,gramacy2020surrogates}) and compute $G^*_\perp$ to obtain a reduced order model of $G$.
We focus on reducing the inputs and/or outputs by minimizing \eqref{eq:error_decomp} over $U_r$ and/or $V_s$, which is equivalent to
\begin{equation}\label{eq:optimal_UrVs_error}
	\max\;\Tr\left(V_s^\top \Cov\left(\bbE[\left.G(X) \right| U_r^\top X ]\right) V_s\right).
\end{equation}
Let us look at some computational aspects of \eqref{eq:optimal_UrVs_error}. Consider the goal-oriented case with fixed $U_r$. The matrix $V_s^*(U_r)$ that maximizes \eqref{eq:optimal_UrVs_error} spans the $s$-dimensional subspace which most effectively captures the output variations caused by parameter of interest $X_r=U_r^\top X$. It can be computed as the matrix containing the $s$ dominant eigenvectors of $\Cov(\bbE[ G(X) |\, U_r^\top X ])$, see \Cref{Lem:TruncatedSVD} bellow. Without input dimension reduction $r=d$, this would correspond to the PCA solution which consists of the dominant eigenvectors of $\Cov(G(X))$ \cite{hotelling1933PCA}. Replacing $G(X)$ by its conditional expectation $\bbE[ G(X) | U_r^\top X ]$ is a natural modification to couple the output dimension reduction with a prescribed input subspace $U_r$. However, computing the covariance of the conditional expectation is challenging. It poses a common problem in global sensitivity analysis and requires dedicated algorithms such as the pick-freeze method \cite{sobol1993PF}.
The alternative goal-oriented problem with fixed $V_s$ gives us a matrix $U_r^*(V_s)$ that extracts the inputs which best explain the projected output $V_r^\top G(X)$. This problem is even more difficult to solve as there exists no closed form expression of $U_r^*(V_s)$ in general. The same computational difficulties affect the problem of finding the optimal coupled matrix pair $(U_r^*,V_s^*)$ that maximizes \eqref{eq:optimal_UrVs_error}. In the following section, we derive a gradient-based bound that provides a feasible optimization procedure to identify the reduced subspaces.

\subsection{Derivative-based Bounds} \label{subsec:GradientBasedBound}

In this subsection, we derive derivative-based bounds for the $L^2$-error \eqref{eq:error_decomp} which permit more manageable optimization problems compared to \eqref{eq:optimal_UrVs_error}.
To derive the bounds, we invoke the Poincar\'{e} inequality \cite{poincare1890equations} and a Cram\'{e}r-Rao-like inequality \cite{rao1945AccuracyAttainable} for the input random variable $X$.

\begin{definition}[Poincar\'{e} and Cram\'{e}r-Rao inequality]\label{def:PI}
	Given a random vector $X$ taking values in $\bbR^d$, we denote the smallest and the largest constants by $\CC(X)\geq0$ and $c(X)\geq0$, respectively, such that
	\begin{equation}\label{eq:PI}
		c(X)\left\| \bbE\left[\nabla f(X)\right] \right\|^2
		\leq 
		\bbE\big[ \left(f(X)-\bbE[f(X)]\right)^2\big] 
		\leq
		\CC(X)\, \bbE\big[\left\| \nabla f(X)\right\|^2\big], 
	\end{equation}
	holds for any continuously differentiable function $f\colon\bbR^d\rightarrow \bbR$. 
	If $ \CC(X)<\infty$, the right inequality is called the \emph{Poincar\'{e} inequality} with constant $ \CC(X)$ and, if $c(X)>0$, the left inequality is called the \emph{Cram\'{e}r-Rao inequality} with constant $ c(X)$.
\end{definition}

It is well known that a Gaussian random vector $X\sim\CN(\mu,\Sigma)$ on $\bbR^d$ satisfies \eqref{eq:PI} with $ c(X)=\lambda_{\min}(\Sigma)$ and $ \CC(X)=\lambda_{\max}(\Sigma)$, see \Cref{lem:CRI} and \cite{ane2000inegalites}. More details regarding the Poincar\'{e} inequality can be found in \cite{bakry2014analysis, zahm2022certified} and references therein. For the Cram\'{e}r-Rao inequality see the discussion in \Cref{app:ProofCRI}. We defer to Proposition \ref{prop:control_constants} for sufficient conditions so that $ c(X)>0$ and $ \CC(X)<\infty$ hold. We now state our main result.
The proof can be found in \Cref{app:ProofErrBound}.

\begin{theorem}\label{thm:errbound}
	Given a random vector $X$ taking values in $\bbR^d$, we let $\overline{\CC}(X)\geq0$ and $\overline{c}(X)\geq0$ be the smallest and the largest constant, respectively, such that
	\begin{align}
		\CC(X_\perp|\, X_r = x_r) &\leq \overline{\CC}(X) ,  \label{eq:subspacePI} \\
		c(X_\perp|\, X_r = x_r) &\geq \overline{c}(X) , \label{eq:subspaceCramer} 
	\end{align}
	holds for all $x_r\in\bbR^r$ and all $U_r\in\bbR^{d\times r}$ with orthonormal columns, where $X_r=U_r^\top X$ and $X_\perp=U_\perp^\top X$ with $U_\perp\in\bbR^{d\times (d-r)}$ being an orthogonal complement to $U_r$. In other words, we require every conditional distribution of $X_\perp$ to satisfy the Poincar\'{e} and Cram\'{e}r-Rao inequality. 

	Let $G\colon\bbR^d\rightarrow\bbR^m$ be a continuously differentiable function such that $\bbE[\|\nabla G(X)\|_F^2] < \infty$, where $\nabla G(X)\in\bbR^{m\times d}$ denotes the Jacobian of $G$ and $\|\cdot\|_F$ is the Frobenius norm.
	Then, for any $U_r\in \bbR^{d\times r}$ and $V_s\in\bbR^{m\times s}$ with orthonormal columns, we have
	\begin{align}
		\bbE\left[\|G(X)- G^*(X)\|^2\right] &\leq \overline{\CC}(X) \Bigl(  \bbE\left[\| \nabla G(X) \|_F^2\right] - \bbE\bigl[\| V_s^\top \nabla G(X) U_r \|_F^2\bigr] \Bigr), \label{eq:error_bound} \\
		\bbE\left[\|G(X)- G^*(X)\|^2\right] &\geq \overline{c}(X) \left( \bigl\|\bbE\left[ \nabla G(X) \right]\bigr\|_F^2 - \left\| V_s^\top \bbE\bigl[\nabla G(X) \bigr] U_r \right\|_F^2 \right) ,\label{eq:error_lowerbound}
	\end{align}
	with $G^*$ defined as in \Cref{thm:optimalG}.
\end{theorem}

Instead of solving the computationally expensive optimization problem  \eqref{eq:optimal_UrVs_error}, we can now optimize the upper bound \eqref{eq:error_bound} of the $L^2$-error to find the subspaces, meaning
\begin{equation}\label{eq:optimal_bound}
	\max\;
	\bbE\left[\left\| V_s^\top \nabla G(X) U_r \right\|_F^2\right] .
\end{equation}
We will show in the next subsection that maximizing over \emph{either} $U_r$ or $V_s$ yields closed-form solutions that can be computed as dominant eigenvectors of certain diagnostic matrices.
There is, however, no closed-form solution for the joint maximization over $U_r$ \emph{and} $V_s$.
Nonetheless, since the Stiefel manifold $\text{St}(p,q) = \{A\in\bbR^{p\times q}\;|\;A^\top A = I_q\}$ is compact \cite{absil2009optimization} and the objective function is continuous with respect to $U_r$ and $V_s$, the joint maximization problem has a solution. 
We propose a natural heuristic for its computation in the next subsection.

The following proposition gives sufficient conditions on the density $\pi_X$ which ensures $\overline{c}(X) >0$ and $ \overline{\CC}(X)<\infty$. 

\begin{proposition}\label{prop:control_constants}
	The following holds:
	\begin{enumerate}
		\item Assume that $\supp(\pi_X)$ is convex and that there exists $\rho>0$ so that 
		\begin{equation}\label{eq:PI_assumption}
			-\nabla^2\log\pi_X(x) \succeq \rho\, I_d
		\end{equation}
		holds for any $x\in\bbR^d$. Then $\overline{\CC}(X)\leq1/\rho$.
		\item Assume that $\supp(\pi_X) = \bbR^d$ and that there exists $\eta>0$ so that
		\begin{equation}\label{eq:CRI_assumption}
			-\nabla^2\log\pi_X(x) \preceq\eta\, I_d
		\end{equation}
		holds for any $x\in\bbR^d$. Then $\overline{c}(X)\geq1/\eta$.
	\end{enumerate}
\end{proposition}
\begin{proof}
	For the proof of 1 see \cite{zahm2022certified}, and for 2 see \Cref{lem:control_subspaceCRI}. 
\end{proof}

While assumption \eqref{eq:PI_assumption} is a classical strict log-concavity assumption on $\pi_X$, the assumption \eqref{eq:CRI_assumption} means that the density $\pi_X$ has to lie above some scaled Gaussian density (see \Cref{rem:interpret_CRI_assumption}). 
Assuming access to $\overline c(X)$, the lower-bound \eqref{eq:error_lowerbound} can be evaluated with no additional cost after computing $U_r$ and/or $V_s$ and allows us to certify the error made by optimizing the upper bound. 
For Gaussian priors, $\overline c(X)$ can be shown to be the smallest eigenvalue of the prior covariance matrix. For non-Gaussian priors, one may normalize the prior and consider an inference problem in a transformed coordinate system in order to obtain the values for both constants $\overline\CC(X)$ and $\overline c(X)$ \cite{cui2022prior}.

\begin{remark}[Preconditioning]\label{rem:GaussianPrior}
	Using Proposition \ref{prop:control_constants}, one can show for $X\sim\mathcal{N}(\mu,\Sigma)$ that
	$\overline{\CC}(X) = \lambda_{\max}(\Sigma)$ and $\overline{c}(X) =  \lambda_{\min}(\Sigma)$. Using a preconditioned vector $\overline X =  \Sigma^{-1/2}(X-\mu) \sim\CN(0,I_d)$, we get sharper bounds for~\eqref{eq:error_bound} and~\eqref{eq:error_lowerbound} with $\overline{\CC}(\overline X)=\overline{c}(\overline X) = 1$. 
	Appling \Cref{thm:errbound} on $\overline X\mapsto G(\Sigma^{1/2}\overline X + \mu)$ yields
	\begin{equation*}
		\bbE\left[\|G(X)- G^*(X)\|^2\right]
		\leq \bbE\big[\| \nabla G(X)\Sigma^{1/2} \|_F^2\big] - \bbE\bigl[\| \overline V_s^\top \nabla G(X)\Sigma^{1/2} \overline U_r \|_F^2\bigr],
	\end{equation*}
	where we use $G^*(X) = V_s^\top \bbE[ G(X) | X_r ]+  (I_m-V_s V_s^\top)\,\bbE[G(X)] $ with reduced parameter $ X_r = (  \overline U_r^\top\Sigma^{-1/2} ) (X-\mu) $. 
	In practice, even for non-Gaussian $X$, it can be beneficial to apply the change of variable $\overline X= \Cov(X)^{-1/2}(X-\bbE[X])$ and compute the dimension-reduction with the isotropic random variable $\overline X$ instead of $X$.
\end{remark}

\begin{remark}[Equality and Closed-form Solution for Affine Models]\label{rem:equality_affine_model}
	For affine models of the form $G(x)=a + Mx$ with $a\in\bbR^{m}, M\in\bbR^{m\times d}$, the Jacobian $\nabla G(X) = M$ is constant.  If we further assume $X\sim\CN(0,I_d)$, so that  $\overline{\CC}(X) = \overline{c}(X) =1 $, the upper and lower bound coincide with the $L^2$-error
	\begin{equation*}
		\bbE\left[\|G(X)- G^*(X)\|^2\right]
		= \left\| M \right\|_F^2 - \| V_s^\top M U_r \|_F^2.
	\end{equation*}
	Thus, the solution to \eqref{eq:optimal_bound} is actually the same as the one of \eqref{eq:optimal_UrVs_error}.
	By the Eckart-Young-Mirsky theorem, the optimal $U_r$ and $V_s$ minimizing the right hand side have columns consisting of the dominant right and left eigenvectors of $M$, respectively, meaning that $U_r^* = [u_1 \ldots u_r]$ and $V_s^* = [v_1 \ldots v_s]$ where
	\begin{equation*}
		M = \sum_{i=1}^{\min(m,d)} \lambda_i v_i u_i^\top
	\end{equation*}
	is the singular value decomposition of $M$. 
\end{remark}

\subsection{Minimizing the Error Bound}\label{subsec:AlternEIG}

We define the diagnostic matrices
\begin{align}
	H_X(V_s) &= \bbE\left[ \nabla G(X)^\top V_s  V_s^\top \nabla G(X) \right] \in\bbR^{d\times d}, \label{eq:HX}\\
	H_Y(U_r) &= \bbE\left[ \nabla G(X) U_r  U_r^\top \nabla G(X)^\top \right]\in\bbR^{m\times m}.\label{eq:HY}
\end{align}
Then the objective function in \eqref{eq:optimal_bound} can be written as
\begin{equation}\label{eq:errbound_as_trace}
	\bbE\left[\| V_s^\top \nabla G(X) U_r \|_F^2\right]
	= \Tr\left(V_s^\top H_Y(U_r)V_s\right) = \Tr\left(U_r^\top H_X(V_s)U_r\right).
\end{equation}
To find the optimal subspaces that minimize the objective above, we recall the following property on Hermitian matrices (see e.g. Corollary 4.3.39 \cite{horn2012matrix}).
\begin{lemma}[Variational characterization of eigenvalues of Hermitian matrices]\label{Lem:TruncatedSVD}
	Let $H\in\bbR^{d\times d}$ be a symmetric positive definite matrix with eigenpairs $(\lambda_i,w_i)$ meaning $Hw_i = \lambda_i w_i$, where $\lambda_{i+1} \leq \lambda_{i}$ and $\|w_i\|_2=1$ for $i=1,\dots,d$. Then for any $r<d$ we have
	\begin{equation*}
		\max_{\substack{W_r\in\bbR^{d\times r}  \\ W_r^\top W_r=I_r}} \Tr\left(W_r^\top H W_r\right) = \sum_{i=1}^r \lambda_i,
	\end{equation*}
	where a solution is given by $W_r = [w_1,\ldots,w_r]$.
\end{lemma}

Applying \Cref{Lem:TruncatedSVD} to the second term $\Tr\left(V_s^\top H_Y(U_r)V_s\right)$ in \eqref{eq:errbound_as_trace}, we obtain the solution $V_s^*(U_r)$ to the goal-oriented problem \eqref{eq:optimal_bound} with fixed $U_r$. The solution requires computing the eigendecomposition $H_Y(U_r) = \sum_{i=1}^m \lambda_i^Y(U_r) v_i v_i^\top$ and assembling
\begin{equation}\label{eq:Vs}
	V_s^*(U_r) = [v_1,\hdots,v_s] .
\end{equation}
In a similar fashion, the third term in \eqref{eq:errbound_as_trace} suggests that the solution $U_r^*(V_s)$ to the goal-oriented problem \eqref{eq:optimal_bound} with fixed $V_s$ can be obtained by computing the eigendecomposition $H_X(V_s) = \sum_{i=1}^d \lambda_i^X(V_s) u_i u_i^\top$ and assembling
\begin{equation}\label{eq:Ur}
	U_r^*(V_s) = [u_1,\hdots,u_r].
\end{equation}
By \eqref{eq:error_bound}, we also get the following error bounds for the two goal-oriented problems
\begin{align*}
	\bbE\left[\|G(X)- G^*(X)\|^2\right]
	&\leq \overline{\CC}(X) 
	\Big( \bbE\left[\|\nabla G(X)\|_F^2\right] - \sum_{i=1}^s \lambda_i^Y(U_r) \Big),\\
	\bbE\left[\|G(X)- G^*(X)\|^2\right]
	&\leq \overline{\CC}(X) 
	\Big( \bbE\left[\|\nabla G(X)\|_F^2\right] - \sum_{i=1}^r \lambda_i^X(V_s) \Big).
\end{align*}
For the joint minimization over $(U_r,V_s)$ for the coupled problem \eqref{eq:optimal_bound}, there is no closed form solution in general. To approximate the solution, we propose the iterative \emph{alternating eigendecomposition} scheme
\begin{equation}
	\begin{split}\label{eq:AlternEIG}
		U_r^{k+1} &= U_r^*(V_s^{k}), \\
		V_s^{k+1} &= V_s^*(U_r^{k+1}),
	\end{split}
\end{equation}
where $U_r^*(\cdot)$ and $V_s^*(\cdot)$ are defined as in \eqref{eq:Ur} and \eqref{eq:Vs}.
If $\{(U_r^{k},V_s^{k})\}_{k\geq1}$ converges towards a limit point $(U_r^*,V_s^*)$ then, by construction, this point will be a stationary point of the objective function. 
In practice, the diagnostic matrices can be computed using Monte Carlo estimators and the operational algorithm is formalized in \Cref{alg:AlternEIG}.

\begin{algorithm}
	\caption{Alternating eigendecomposition for coupled input and output dimension reduction.\label{alg:AlternEIG}}
	\begin{flushleft}
	\hspace*{\algorithmicindent} \textbf{Input:} $s,\;r\in\bbN$; $V_s^{0}\in\bbR^{m\times s}$; $\VG_i = \nabla G(x^{(i)})$ where $x^{(i)}\overset{i.i.d.}{\sim}\pi_X$
	\end{flushleft}
\begin{algorithmic}[1]
	\STATE{Set $k = 0$}
	\WHILE{$k\leq \text{maximal iteration}$}
	\STATE{Assemble $\widehat H_X(V_s^{k}) = \frac{1}{M}\sum_{i=1}^M \VG_i^\top V_s^k  {V_s^k}^\top \VG_i$}
	\STATE{Compute SVD of $\widehat H_X(V_s^{k}) = U \Lambda U^\top $}
	\STATE{Set $U_r^{k} = [u_1,\hdots,u_r]$}
	\STATE{Assemble $\widehat H_Y(U_r^{k}) = \frac{1}{M}\sum_{i=1}^M \VG_i U_r^{k}  {U_r^{k}}^\top \VG_i^\top$}
	\STATE{Compute SVD of $\widehat H_Y(U_r^{k}) = V \Lambda V^\top $}
	\STATE{Set $V_s^{k+1} = [v_1,\hdots,v_s]$}
	\STATE{$k = k + 1$}
	\ENDWHILE
	\RETURN $U_r$, $V_s$
\end{algorithmic}
\end{algorithm}

\begin{remark}[Implicit diagnostic matrices]\label{rem:matvec}
	\Cref{alg:AlternEIG} requires computing eigencompositions of the matrices $\widehat H_X(V_s)$ and $\widehat H_Y(U_r)$ which, when $d\gg1$ or $m\gg1$, can be numerically costly to assemble and to store. To remedy this, one can compute the eigencompositions using iterative algorithms (e.g., Krylov methods) which only require access to matrix-vector product operations of the form
	\begin{align*}
		v&\mapsto \widehat H_Y(U_r)v = \frac{1}{M}\sum_{i=1}^M \nabla G(X^{(i)}) \left( U_r  U_r^\top   \left(\nabla G(X^{(i)})^\top v \right) \right) ,\\
		u&\mapsto \widehat H_X(V_s)u = \frac{1}{M}\sum_{i=1}^M \nabla G(X^{(i)})^\top \left( V_s  V_s^\top \left(\nabla G(X^{(i)}) u\right)\right) .
	\end{align*}
	Computing these matrix-vector products requires only evaluating the linear-tangent operator $u\mapsto \nabla G(X^{(i)}) u$ and its adjoint $v\mapsto\nabla G(X^{(i)})^\top v$,
	which have a computational complexity of roughly two forward simulations \cite{plessix2006review,griewank2008evaluating}.
	Since the gradient-based dimension reduction methods are intended as a preprocessing step for the analysis of $G$ (e.g., optimization, calibration or parameter estimation), access to gradients is often required by the prospective analysis itself. Finally, if closed-form gradients are not available, automatic differentiation or gradient approximation methods, such as finite differences, may be used. 
	Both the accuracy of the gradient evaluations and the Monte Carlo estimation of the diagnostic matrices will affect the obtained subspaces $U_r$ and $V_s$. A detailed error analysis is left for future work. 
\end{remark}

\begin{remark}\label{rem:jointDR}
In a recent paper \cite{baptista2022JointDR}, the authors proposed a joint dimension reduction method for the parameter and data space based on similar diagnostic matrices. Using an information theoretic approach, they derive an upper bound for the posterior error $\bbE_Y[\KL(\pi_{X|Y}\|\widetilde\pi_{X|Y})]$ which is minimized by $U_r^*$ and $V_s^*$ consisting of the dominant eigenvectors of
\begin{align*}
	H_X &= \bbE\left[\nabla G(X)^\top \nabla G(X) \right], \label{eq:def_HX}\\
	H_Y &= \bbE\left[\nabla G(X) \nabla G(X)^\top \right],
\end{align*}
respectively.
Their results correspond to our goal-oriented solutions with $V_s^* = V_s^*(I_d)$ and $U_r^* = U_r^*(I_m)$.
Contrary to our error bound, their results depend on the logarithmic Sobolev constant of the joint distribution $\pi_{XY}$ whose existence is generally difficult to establish.
\end{remark}

\section{Application to Goal-oriented Coordinate Selection}\label{sec:Goal}

\subsection{Goal-Oriented Bayesian Optimal Experimental Design}\label{subsec:GoalBOED}

Consider the problem of finding an optimal subset of $s \leq m$ sensors out of $m$ candidate locations that maximizes the information gained about the parameter $X$ that is to be inferred.
The vector $Y=G(X)+\eta$ gathers all possible measurements at the $m$ candidate locations.
Data collected at select sensor locations $\tau\subseteq \{1,\ldots,m\}$, $|\tau|=s$ are denoted by $Y_\tau=V_\tau^\top Y$, where $V_\tau = [e_\tau^m]\in \bbR^{m\times s}$ contains the canonical basis vectors $e_{i}^m\in\bbR^m$ indexed by $\tau$.
Later, we will propose an alternative approach for experimental design where the reduced data is characterized by a matrix $V_s\in\bbR^{m\times s}$ with orthonormal columns. We will explain the benefits of considering this relaxed design space later and how to associate sensor locations to the design matrix $V_s$. 
For any $V_\star$, which can be a coordinate selection design $V_\star=V_\tau$ or a relaxed orthogonal design $V_\star=V_s$, we have the posterior density of $X$ knowing $Y_\star = V_\star^\top Y$ given as
\begin{equation}
	\pi_{X|V_\star^\top Y}(x|y_\star) 
	\propto \exp\left(-\tfrac{1}{2\sigma^2} \left\|y_\star-V_\star^\top G(x)\right\|^2\right) \pi_X(x).
	\label{eq:Likelihood_subspace}
\end{equation}
To quantify the information content of an experimental design $V_\star$, we use the expected information gain (EIG) defined as the KL-divergence between posterior and prior distribution in expectation over the data
\begin{equation}\label{eq:def_EIG}
	\Phi(V_\star) = \bbE_{Y}\left[ \KL\left(\pi_{X|V_\star^\top Y}(\,\cdot\,|V_\star^\top Y)\|\,\pi_{X}\right)\right].
\end{equation}
The divergence reflects the extent to which the posterior distribution differs from the prior due to observed data and, thereby, represents the information gain of the design $V_\star$. 
The EIG is equivalent to the mutual information between $X$ and $Y$ given as $I(X,Y) = \KL(\pi_{X,Y}\|\pi_X \pi_Y)$, i.e., the KL-divergence from the product of marginal densities to their joint density. This perspective can be useful for models with inaccessible likelihood functions~\cite{kraskov2004kNN_MI}. 
For a linear Gaussian inference problem, optimizing the EIG corresponds to minimizing the determinant of the posterior covariance matrix, also referred to as the D-optimality criterion. It can be interpreted as minimizing the volume of the uncertainty ellipsoid around the maximum a-posteriori estimator of the parameter $X$ \cite{alexanderian2021Review}.

Estimating the EIG poses an even greater challenge than solving the Bayesian inference problem as it requires solving the inverse problem itself as a sub-problem. 
The EIG consists of an expectation over both the posterior and data distributions, necessitating an expensive nested Monte Carlo estimator \cite{ryan2003NMC_EIG,beck2018DLIS_EIG}. 
Efficient EIG estimation is an active research area, and various approaches exist. Some rely on constructing tractable approximations to the posterior distribution using measure transport \cite{scheichl2024TransportOED,baptista2024MeasureTransportOED,li2024expected} or variational inference \cite{foster2019variationalEIG}. Others employ different types of surrogate models \cite{huan2013simulation,chen2023derivativeNN,aretz2023greedyOED}. 
A widely used strategy is to resort to Gaussian linear models for which a closed-form expression of the EIG exists \cite{alexanderian2016fast,chen2021GoalEIG,alexanderian2018GoalOED}. 
In addition to the challenging EIG estimation, identifying the optimal set of sensor placements requires solving a combinatorial optimization problem with $\binom{m}{s}$ possible combinations. This NP-hard problem is intractable even for moderate dimensions. Hence, sub-optimal solutions are commonly identified using greedy methods \cite{krause2008GreedySensorsGP,cohen2018GreedySensorsPBDW,chen2021GoalEIG} or by solving continuous relaxations of the problem with sparsity constraints~\cite{haber2008IllposedOED,alexanderian2018GoalOED}.

We consider, in particular, goal-oriented BOEDs, where the final target is to infer the parameter in a lower-dimensional linear subspace of interest. This quantity of interest $X_r=U_r^\top X$ is characterized by a prescribed linear 'goal' operator $U_r\in\bbR^{d\times r}$.
For example, if the PDE model describes a spatio-temporal field with $X$ as the initial state, the initial state in a localized region of the domain could be a parameter of interest.
Tailoring the experimental design to maximize the information gained with respect to the parameter of interest allows for more efficient and accurate inference from experimental data. 
For a given parameter of interest $X_r$, we seek the optimal design $V_\star$ that maximizes the goal-oriented EIG (see \eqref{eq:objective_goalEIG})
\begin{equation}\label{eq:goalEIG}
	\max_{V_\star}\; \Phi(V_\star | U_r) = \bbE_{Y}\left[ \KL\left(\pi_{X_r|V_\star^\top Y}(\,\cdot\,|V_\star^\top Y)\|\,\pi_{X_r}\right)\right] , 
\end{equation}
where $\pi_{X_r|V_\star^\top Y}$ and $\pi_{X_r}$ are the marginalized densities. 
The goal-oriented BOED problem faces the same obstacles as the non-goal-oriented one, namely the expensive EIG estimation and combinatorial optimization problem. 
We avoid both challenges by maximizing a lower bound for the EIG. Although our method requires the assumption of a Gaussian likelihood, it allows for non-linear forward models.

Combining \Cref{Lem:L2vsEIG} with \Cref{thm:errbound} permits us to optimize a lower bound of the goal-oriented EIG by maximizing $\Tr\left(V_\star^\top H_Y(U_r) V_\star\right)$. 
In the coordinate selection case with $V_\star=V_\tau$, this reduces to solving the problem
\begin{equation}\label{eq:GoalOED_tau}
	\max_{\substack{V_\tau = [e_\tau]\in\bbR^{m\times s}\\ \tau\subset\{1,\ldots,m\},\; |\tau|=s}}\; \Tr\left(V_\tau^\top H_Y(U_r) V_\tau\right) = \sum_{i\in\tau^*} H_Y(U_r)_{ii} ,
´\end{equation}
where the solution $\tau^*\subseteq\{1,\ldots,m\}$ is the index set of the $s$ largest diagonal entries of $H_Y(U_r)$.
While our solution $\tau^*$ may be sub-optimal due to the maximization of a lower bound, it has the advantage of being computable without the expensive evaluation of the EIG nor the need to solve a combinatorial optimization problem. 

For orthogonal designs $V_\star=V_s$, we obtain with \Cref{Lem:TruncatedSVD} the solution
\begin{equation}\label{eq:GoalOED_subspace}
	\max_{\substack{V_s\in\bbR^{m\times s}\\V_s^\top V_s = I_m}}\; \Tr\left(V_s^\top H_Y(U_r) V_s\right) = \sum_{i=1}^s \lambda_i^Y(U_r) ,
\end{equation}
with $V_s^*(U_r) = [v_1 , \dots, v_s]$ and $(\lambda_i^Y,v_i)$ being the $i$-th largest eigenpair of $H_Y(U_r) = \sum_{i=1}^m \lambda_i^Y(U_r) v_iv_i^\top$ with $\lambda_{i+1}^Y(U_r) \leq \lambda_i^Y(U_r)$.
Given that~\eqref{eq:GoalOED_subspace} is a relaxation of \eqref{eq:GoalOED_tau}, we expect $V_s^*(U_r)$ to yield better results compared to $V_{\tau^*}$, meaning $ \sum_{i=1}^s \lambda_i^Y(U_r) \geq \sum_{i\in\tau^*} H_Y(U_r)_{ii}$. To associate sensor locations with the optimal subspace $V_s^*(U_r)$, we propose to employ the empirical interpolation method (EIM) \cite{maday2004EIM,maday2007MagicPoints}.

The EIM takes as input a set of basis vectors  $v_1,\ldots,v_s\in\bbR^m$ and returns a set of indices $\{i_1,\ldots,i_s\}\subseteq\{1,\ldots,m\}$ corresponding to the empirically-best interpolation points for the given basis. Combining the EIM with dimension reduction methods for sensor placement has been considered, for example, in \cite{giraldi2018optimal,mula2018sensorEIM}.
By using the EIM on the columns of $V_s^*(U_r)$, we can thus associate sensor locations $\widetilde\tau = \{i_1,\ldots,i_s\}$ to the orthogonal design matrix $V_s^*(U_r)$. 
Compared to $\tau^*$, the EIM-based sensors  $\widetilde\tau$ are less optimal with respect to the goal-oriented EIG lower bound \eqref{eq:GoalOED_tau}. 
However, our numerical experiments show that the EIM-based sensors represent a viable alternative with similarly good performance as $\tau^*$ in terms of the \emph{true} goal-oriented EIG rather than the lower bound. In contrast to $\tau^*$, the EIM-based sensors $\widetilde\tau$ tend to be less localized in our experiments.
In the numerical experiments, we also apply the EIM on PCA bases to obtain sensor placements $\widetilde\tau^{\PCA}$, which we will refer to as PCA-EIM sensors.

\subsection{Goal-Oriented Global Sensitivity Analysis}\label{subsec:GoalGSA}

Global sensitivity analysis aims to assign an importance value to each set of input variables, reflecting its contribution to the output variations. It is often used for factor prioritization and/or fixing in complex models. Factor prioritization identifies the subset $\tau\subseteq\{1,\ldots,d\}$, $|\tau|=r$ of input parameters causing the most output variation so that the parameters $X_\tau=(X_i)_{i\in\tau}$ can be prioritized e.g., in uncertainty reduction. Factor fixing, on the other hand, identifies the input set $\tau$ with the least influence on the output. Thus, $X_\tau$ can be fixed at a nominal value to reduce the complexity of, e.g., parameter estimation problems.
For a scalar-valued square-integrable function $f\colon\bbR^d\rightarrow\bbR$ with independent inputs $X_1,\ldots,X_d$, commonly used sensitivity indices are the closed Sobol' indices and the total Sobol' indices defined as 
\begin{equation}\label{eq:def_Sobol_scalar}
	S^{\cl}_\tau = \frac{\Var\left(\bbE\left[\left.f(X)\right|X_\tau\right]\right)}{\Var\left(f(X)\right)}
	\qquad \text{and} \qquad 
	S^{\tot}_\tau = 1-\frac{\Var\left(\bbE\left[\left.f(X)\right|X_{-\tau}\right]\right)}{\Var\left(f(X)\right)}, 
\end{equation}
respectively, where $X_{-\tau}=(X_i)_{i\notin\tau}$. 
The closed Sobol' index quantifies the portion of output variance attributed to $X_\tau$. A higher value implies a stronger influence of $X_\tau$ on the output. The total Sobol' index encompasses output variation due to $X_\tau$ along with variations caused by interactions of $X_\tau$ with any of the remaining input variables. As a result, it holds that $S^{\cl}_\tau \leq S^{\tot}_\tau$ and $1 = S^{\tot}_\tau + S^{\cl}_{-\tau}$. A lower value in total Sobol' index suggests that the parameters indexed by $\tau$ have little direct nor indirect impact on the output.

We consider, in particular, goal-oriented GSA where some outputs of interest $V_s^\top G(X)$ are prescribed. 
We use definitions of Sobol' indices for vector-valued functions in \cite{gamboa2014multidim} and
generalize them to the goal-oriented case as follows
\begin{equation}
	\begin{split}
		S^{\cl}(U_\tau|V_s) 
		= \frac{\Tr\left( \Cov\left( \bbE\left[\left.V_s^\top G(X)\right| U_\tau^\top X\right] \right)\right) }{\Tr\left( \Cov\left(V_s^\top G(X) \right) \right) }, \\
		S^{\tot}(U_\tau|V_s) 
		= 1-\frac{\Tr\left(\Cov\left(\left. V_s^\top G(X)\right| U_{-\tau}^\top X \right)\right)}{\Tr\left(\Cov\left( V_s^\top G(X) \right)\right)} .
	\end{split}
\end{equation}
For $V_s\in\bbR^{m\times 1}$ being a vector, the provided indices correspond to the ones for a scalar-valued output \eqref{eq:def_Sobol_scalar}.

Using the Poincar\'{e} and Cram\'{e}r-Rao inequality, we now derive derivative-based bounds for the goal-oriented Sobol' indices involving similar derivative-based quantities as in \Cref{thm:errbound}.

\begin{proposition}\label{prop:GradvsGSA}
	With the above notations, for any $U_\tau = [e_\tau^d]\in\bbR^{d\times r}$ and $V_s\in\bbR^{m\times s}$ with orthogonal columns, we have
	\begin{align}
		1- \overline{\CC}(X) \frac{ \Tr\left( U_{-\tau}^\top H_X(V_s) U_{-\tau}\right)}{\Tr\left(\Cov\left(V_s^\top G(X)\right)\right)}
		&\leq  S^{\cl}(U_{\tau}|V_s) \leq
		1- \overline{c}(X) \frac{ \left\|V_s^\top\bbE\left[ \nabla G(X) \right] U_{-\tau} \right\|_F^2  }{\Tr\left(\Cov\left(V_s^\top G(X)\right)\right)} , \label{eq:bound_closedSobol}\\
		\overline{c}(X)  \frac{  \left\|V_s^\top\bbE\left[ \nabla G(X) \right] U_{\tau} \right\|_F^2 }{\Tr\left(\Cov\left(V_s^\top G(X)\right)\right)}  
		&\leq S^{\tot}(U_{\tau}|V_s) \leq \overline{\CC}(X) \frac{  \Tr\left( U_{\tau}^\top H_X(V_s) U_{\tau}\right) }{\Tr\left(\Cov\left(V_s^\top G(X)\right)\right)},  \label{eq:bound_totalSobol}
	\end{align}
	where $U_{-\tau} = [e_{-\tau}^d]\in\bbR^{d\times (d-r)}$ is the orthogonal complement to $U_\tau$.
\end{proposition}

The proof can be found in \Cref{app:ProofGradvsGSA}. 
Similarly to the BOED problem, GSA struggles with an objective function that is expensive to evaluate combined with a difficult combinatorial optimization problem to identify an optimal index set $\tau$ (see discussion in \Cref{sec:Objectives}). On the other hand, optimizing our derivative-based bound instead of the Sobol' indices provides us with a closed-form solution for the optimal index set. That is,
\begin{equation}\label{eq:GoalGSA_tau_closed}
	\min_{\substack{U_{-\tau_{\cl}}\in\bbR^{d\times (d-r)}\\ \tau_{\cl}\subset\{1,\ldots,d\},\; |\tau_{\cl}|=r}}\;\Tr\left(U_{-\tau_{\cl}}^\top H_X(V_s) U_{-\tau_{\cl}}\right) = \sum_{i\notin\tau_{\cl}^*} H_X(V_s)_{ii},
\end{equation}
where the index set $\tau_{\cl}^*$ maximizing the lower bound for the closed Sobol' index is given by the indices corresponding to the largest diagonal entries of $H_X(V_s)$. Correspondingly, we have 
\begin{equation}\label{eq:GoalGSA_tau_total}
	\min_{\substack{U_{\tau_{\tot}} \in\bbR^{d\times r}\\ \tau_{\tot}\subset\{1,\ldots,d\},\; |\tau_{\tot}|=r}}\;\Tr\left(U_{\tau_{\tot}}^\top H_X(V_s) U_{\tau_{\tot}}\right) = \sum_{i\in\tau_{\tot}^*} H_X(V_s)_{ii},
\end{equation}
where the index set $\tau_{\tot}^*$  minimizing the upper bound for the total Sobol' index is given by the indices corresponding to the smallest diagonal entries of $H_X(V_s)$. Note that we can handle index sets of any size for the same computational cost as the case of single-element index sets, given that only the diagonal elements of $H_X(V_s)$ need to be computed in either setting. Moreover, evaluating the respective upper and lower bound of the closed and total Sobol' index allows for complete certification of the induced error. 

\begin{remark}
	Our method generalizes the derivative-based global sensitivity measures (DGSM) introduced in \cite{sobol2009DGSM} which are defined as 
	\begin{equation*}
		\nu_i = \bbE\bigg[\left(\frac{\partial f}{\partial x_i}(X)\right)^2\bigg]
	\end{equation*}
	for a scalar function $f\colon\bbR^d\rightarrow\bbR$. The generalization to vector-valued functions has also been considered in \cite{cleaves2019derivative}.
	Setting $V_s\in\bbR^{m\times 1}$ to be a vector, the diagonal elements $H_X(V_s)_{ii}$ correspond to the DGSM $\nu_i$ of the scalar function $V_s^\top G(X)$. Our approach additionally enables sensitivity measures for goal-oriented vector-valued outputs and is not restricted to single-element sets $\tau = \{i\}$. 
	The authors in \cite{kucherenko2014DGSM_review} also derived the following DGSM-based bounds for the total Sobol' index where $X_i\sim\CN(\mu_i,\sigma_i)$
	\begin{equation*}
		\tfrac{\sigma_i^4}{(\mu_i^2 + \sigma_i^2)V} \omega_i^2
		\leq S_{i}^{\tot}
		\leq \tfrac{\sigma_i^2}{V}\nu_i,
	\end{equation*}
	where $\omega_i = \bbE[\tfrac{\partial f(X)}{\partial x_i}]$ and $V = \Var(f(X))$. For normally distributed $X_i$, we have $\overline c(X) = \overline\CC(X) = \sigma_i^2$ (see \Cref{rem:GaussianPrior}) and thus obtain a slightly improved lower bound $\frac{\sigma_i^2}{V} \omega_i^2 \leq S_{i}^{\tot} \leq \frac{\sigma_i^2}{V}\nu_i$. Note that lower-bounds for the total Sobol' index are generally receiving growing attention in GSA \cite{roustant2020parseval,roustant2019lower}.
\end{remark}

\section{Numerical Experiments}\label{sec:Numerics}

The following numerical experiments are computed with code provided at
\url{https://github.com/qchen95/Coupled-input-output-DR}.

\subsection{Conditioned Diffusion}\label{subsec:Conddiff}

Here, we consider a model for the movement of a diffusive particle acting under a double-well potential in molecular dynamics. The goal is to infer the driving force $dB_t$ applied to the particle given observations of its path $u:[0,T]\rightarrow\bbR$ with final time $T=1$. The particle movement is described by the solution of the stochastic differential equation
\begin{equation*}
	\d u_t = f(u_t)\d t + \d B_t, \qquad u_0=0,
\end{equation*}
with non-linear drift function $f(u) = u(1-u^2)/(1+u^2)$ and $B$ denoting a Wiener process. 
We discretize the differential equation using an Euler-Maruyama scheme with time step $\Delta t = 10^{-2}$ resulting in $100$ uniformly-spaced time steps. Noisy observations $Y$ are taken at each time step so that the forward model $G\colon \mathbb{R}^d \rightarrow \mathbb{R}^m$ is defined as the mapping from the incremental driving force $X$ at $d=100$ times to the particle position $u$ at $m=100$ times. We further assume a standard normal prior $\CN(0, I_d)$ for the driving force and i.i.d. additive observational noise with distribution $\CN(0,0.1)$. \Cref{subfig:conddiff_goal_early,subfig:conddiff_goal_late} illustrate some realizations of the particle path $u$. 

In this problem, we aim to identify optimal subspaces and indices of the inputs and outputs of $G$ that reduce the dimensions of the Bayesian inference problem for inferring $X$ from $Y$. To compute the diagnostic matrices $H_X(V_s)$ \eqref{eq:HX} and $H_Y(U_r)$ \eqref{eq:HY}, we use $M=10000$ prior samples and set $r=s=10$, if not stated otherwise. The alternating eigendecomposition algorithm is run for a fixed number of $10$ iterations and initiated from a random $V_s\in\bbR^{m\times s}$ with orthonormal columns when no goal is specified. We recall that the optimal subspace is given without alternation when only solving for one of $U_r$ or $V_s$. 

\begin{figure}[H]
	\centering
	\begin{subfigure}{0.49\textwidth}
		\centering
		\includegraphics{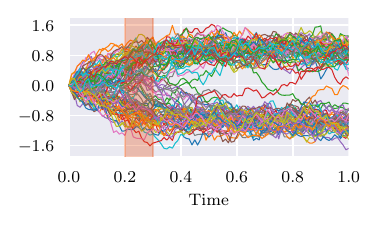}
		\caption{Driving force of interest $U_r^\top X$.\label{subfig:conddiff_goal_early}}
	\end{subfigure}
	\begin{subfigure}{0.49\textwidth}
		\centering
		\includegraphics{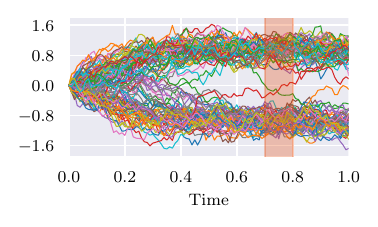}
		\caption{Observations of interest $V_s^\top G(X)$.\label{subfig:conddiff_goal_late}}
	\end{subfigure}
	\par
	\begin{subfigure}{0.49\textwidth}
		\centering
		\includegraphics{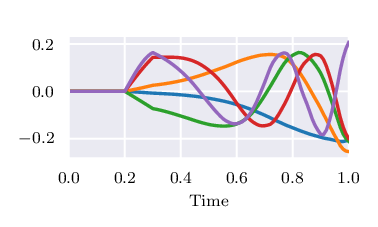}
		\caption{Reduced data $V_s^*(U_r)$.\label{subfig:conddiff_goal_Uearly}}
	\end{subfigure}
	\begin{subfigure}{0.49\textwidth}
		\centering
		\includegraphics{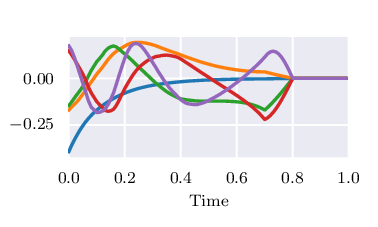}
		\caption{Reduced parameter $U_r^*(V_s)$.\label{subfig:conddiff_goal_Vlate}}
	\end{subfigure}
	\caption{Illustration of the inputs and outputs of interest (i.e., goals) within the specified orange time intervals, and their corresponding goal-oriented solutions for the vectors defining the reduced output and input spaces, respectively. Notice that the solutions correctly reflect the time causality of the problem.}\label{fig:conddiff_goal}
\end{figure}
In \Cref{fig:conddiff_goal}, we observe solutions of the goal-oriented problems. 
The goals $U_r$ and $V_s$, illustrated in \Cref{subfig:conddiff_goal_early} and \Cref{subfig:conddiff_goal_late}, project onto the driving force and, respectively, the particle path observations within the specified orange time intervals. For the goal $U_r$, the resulting observation space bases $V_s^*(U_r)$ \eqref{eq:Vs} in \Cref{subfig:conddiff_goal_Uearly} are zero up to the interval of interest. This behavior can be explained as the driving force within the prescribed time interval does not influence the particle position at earlier times. Similarly, the parameter space bases $U_r^*(V_s)$ \eqref{eq:Ur} in \Cref{subfig:conddiff_goal_Vlate} derived from the goal $V_s$ do not extend beyond the observed time interval. An explanation is that observations of the particle's path do not provide information about the driving force at future times. Thus, our goal-oriented solutions appear to capture the time causality within the inference problem correctly. 

\begin{figure}[htbp]
	\centering 
	\includegraphics{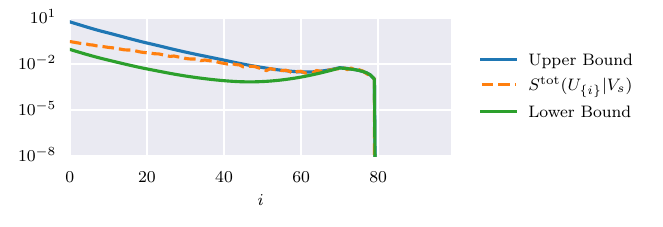}
	\caption{Goal-oriented total Sobol' indices for each index set as compared with the derived derivative-based bounds. The bounds tightly track the true value for the Sobol' indices. The bounds also take into account the time causality of the problem.}
	\label{fig:conddiff_sobol}
\end{figure}
In \Cref{fig:conddiff_sobol}, we present the goal-oriented total Sobol' indices for the output of interest $V_s^\top G(X)$ illustrated in \Cref{subfig:conddiff_goal_Vlate}. For each coordinate $i =1,\dots,d$ of the input parameter, i.e., the driving force at time $i\cdot \Delta t$, we compute $S^{\tot}(U_{\{i\}}|V_s)$ along with the derivative-based bounds in \eqref{eq:bound_totalSobol}. We present single-element input sets $\tau=\{i\}$, $1\leq i \leq d$ for easier illustration. However, we recall that the derived bounds can be computed for arbitrary parameter sets by summing up the elementwise values according to \eqref{eq:GoalGSA_tau_total}.
Both the upper and lower bounds in \Cref{fig:conddiff_sobol} tightly track the total Sobol' indices. 
The bounds also capture the problem's temporal causality, i.e., driving forces at later times exert no influence on the model's output variations during the time interval of interest.

\Cref{fig:conddiff_coupledUV} shows solutions of the coupled dimension reduction problem (top) in comparison with dimension reduction by applying PCA separately on the parameter and data space (bottom) \cite{hotelling1933PCA}. 
The coupled modes of the parameter space focus on approximating the driving force at initial times, reflecting the strong influence of the force at these times on the overall particle path. Contrarily, the modes of the data space focus on approximating the particle path at final times. This phenomenon reflects the fact that the final position of the particle reveals more about the overall driving force. The coupled eigenvectors of both spaces exhibit notable disparities from those derived through PCA.
The PCA modes of the input space reflect the white noise behavior of the driving force. The PCA modes of the observation space are more global as compared to the coupled ones and moreover they do not exhibit awareness of the inverse problem. 
\begin{figure}[htbp]
	\centering
	\begin{subfigure}[b]{0.49\textwidth}
		\centering
		\includegraphics{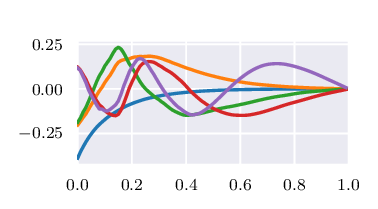}
		\caption{$U_r^{\Coupled}$}
	\end{subfigure}
	\begin{subfigure}[b]{0.49\textwidth}
		\centering
		\includegraphics{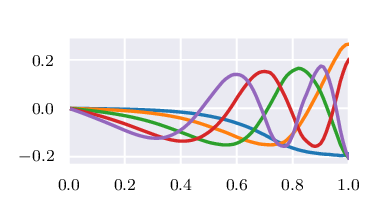}
		\caption{$V_s^{\Coupled}$}
	\end{subfigure}
	\par
	\begin{subfigure}[b]{0.49\textwidth}
		\centering
		\includegraphics{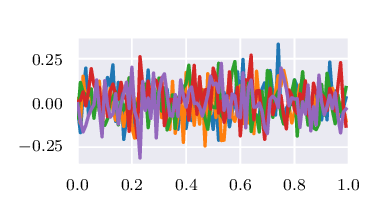}
		\caption{$U_r^{\PCA}$}
	\end{subfigure}
	\begin{subfigure}[b]{0.49\textwidth}
		\centering
		\includegraphics{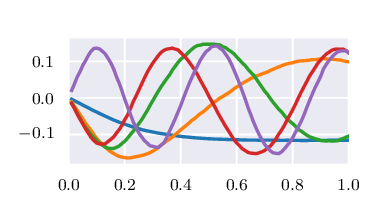}
		\caption{$V_s^{\PCA}$}
	\end{subfigure}
	\caption{Comparison of subspaces from coupled and PCA-based dimension reduction.}
	\label{fig:conddiff_coupledUV}
\end{figure}

\begin{figure}[htbp]
	\centering   
	\includegraphics{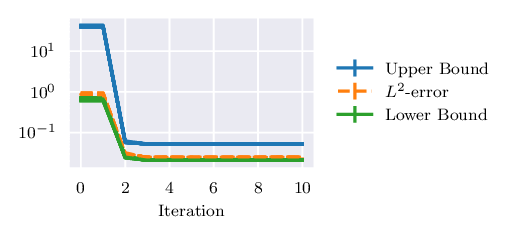}
	\caption{Convergence of of alternating eigendecomposition for $10$ random initializations.}
	\label{fig:conddiff_conv}
\end{figure}

\begin{figure}[htbp]
	\centering
	\includegraphics{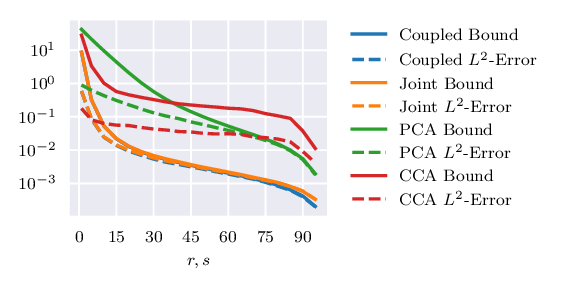}
	\caption{Comparison of model error and upper bound convergence across different dimension reduction methods as subspace dimensions $r,\; s$ increase.}
	\label{fig:conddiff_sr1D}
\end{figure}
Next, we consider the convergence of the alternating eigendecomposition algorithm. \Cref{fig:conddiff_conv} depicts the relative $L^2$-error \eqref{eq:error_decomp} as well as the relative upper \eqref{eq:error_bound} and lower bound \eqref{eq:error_lowerbound} over the iterations of the alternating eigendecomposition for $10$ random initializations of the matrices. 
We observe that minimizing the upper error bound indeed decreases the true $L^2$-error. Moreover, the alternating eigendecomposition converges after just three iterations independent of the initiation.

\Cref{fig:conddiff_sr1D} shows the convergence of the relative $L^2$-error and upper bound with increasing subspace dimensions $r$ and $s$. We compare 
the errors of our coupled dimension reduction method with other dimension reduction methods in the data and parameter spaces such as PCA, canonical correlation analysis (CCA) \cite{hotelling1992CCA} and the joint dimension reduction method \cite{baptista2022JointDR} discussed in \Cref{rem:jointDR}. Our coupled method has very similar error and upper bound values as the joint dimension reduction method. Both decrease much faster with increasing subspace dimensions and do not get stuck on a plateau like CCA.

\subsection{Burgers' Equation}

In this subsection, we demonstrate our goal-oriented BOED framework on the one-dimensional non-linear Burgers' equation
\begin{align*}
\frac{\partial u}{\partial t} + u \frac{\partial u}{\partial x} - D\frac{\partial u}{\partial x^2}= 0,
\end{align*}
with $t\in[0,1]$, $x\in[0,1]$, $D=0.001$ and periodic boundaries. Our parameters are the initial conditions $u_0(x)$. We assume a normal prior distribution for $u_0(x)$ with mean $\mu(x) = \tfrac{1}{\sqrt{2\pi}}\exp(-\tfrac{1}{2}\|0.5-x\|^2)$ and Mat\'{e}rn covariance with correlation length $l=0.1$ and smoothness $\nu=2.5$. Using a uniform space discretization our parameters are $(u_0(x_1), \ldots, u_0(x_d))\in\bbR^d$ with $d=100$. As data $y\in\bbR^m$, we use state observations at the final time $T=0.1$ with i.i.d.\thinspace measurement noise $\eta\sim\CN(0,0.01\, I_m)$ so that $m=100$. Therefore, $G\colon\mathbb{R}^d \rightarrow \mathbb{R}^m$ maps the discretized initial condition $u_0(x)$ to the discretized solution $u(x,T)$ at the final time. The left subplot in \Cref{fig:burgers_prior} shows the prior mean $\mu(x)$ together with several prior samples and also visualizes each of the goals $U_r^{(i)}$, which project the initial conditions $X$ onto the grid points $i,\ldots,i+r$ marked by an orange bar.
The right subplot in \Cref{fig:burgers_prior} illustrates the space-time model solutions $u(x,t)$ with the prior mean as initial condition $u_0(x)=\mu(x)$. 
We consider $r=5$ and the sensor placement for $s=10$ sensors.
To compute the diagnostic matrices, we use $M=1000$ prior samples, and the goal-oriented true EIG \eqref{eq:goalEIG} is computed via the double-loop importance sampling estimator from \cite{beck2018DLIS_EIG} with $30$ inner and $M$ outer loop samples making sure the estimator is sufficiently accurate for the subsequent comparison.

\begin{figure}[htbp]
	\centering
	\begin{subfigure}[t]{0.49\textwidth}
		\centering
		\includegraphics{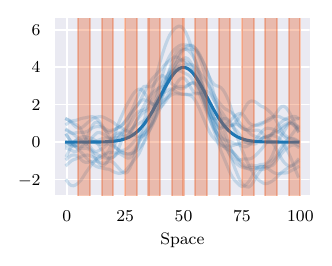}
	\end{subfigure}%
	\begin{subfigure}[t]{0.49\textwidth}
		\centering
		\includegraphics{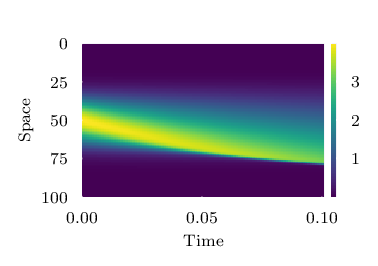}
	\end{subfigure}
	\caption{Left: Illustration of the prior mean $\mu(x)$ (solid line) and samples (transparent lines) of the initial condition parameter with orange bars marking the quantities of interest ${U_r^{(i)}}$. Right: Solutions of Burgers' equation with the prior mean as initial condition, i.e., $u_0(x) = \mu(x)$.}
	\label{fig:burgers_prior}
\end{figure}

In \Cref{fig:EIG_burgers}, we present the goal-oriented EIG for different goals $U_r^{(i)}$ and sensor placement methods. 
For each goal, we select sensors based on four different methods: the optimal sensor placement $\tau^*$ from \eqref{eq:GoalOED_tau}, the goal-oriented EIM-based sensors $\widetilde\tau$, the PCA-EIM sensors $\widetilde\tau^{\PCA}$ and sensors $\tau^F$ based on the Fedorov-Wynn algorithm with a linearized model around the prior mean composed with the linear goal operator \cite{fedorov1972theory,wynn1972results}.
The violin plots also show the EIG values for  $250$ random sensor placements. Both sensors $\tau^*$ and $\widetilde\tau$ almost always outperform the random sensors, i.e., yielding larger EIG. In general, both $\widetilde\tau$ and $\tau^*$ perform similarly well as the Fedorov-Wynn-based sensors $\tau^F$. Their performance surpasses that of $\tau^F$ for the goals $U_r^{(45)}$ and $U_r^{(55)}$, where non-linear effects may significantly compromise the accuracy of the model linearization in the Fedorov-Wynn algorithm. On the other hand, the PCA-EIM sensors $\widetilde\tau^{\PCA}$ are often only as good as average random selections, especially for goals $U_r^{(i)}$ with lower indices $i<45$. This phenomenon can be explained by looking at the actual sensor positions in \Cref{fig:GoalV_sensors}. Since the solution has a wavefront that is moving to the right, the final state has strong features on the right side of the domain. As a result, the PCA-EIM sensors are concentrated on the right side of the domain, regardless of the goal. Thus, the PCA-EIM sensors seem to be more informative for parameters of interests $U_r^{(i)}$ that also lie on the right side of the domain (large $i$) and less so for parameters of interest located on the left side of the domain (small $i$).
Generally, $\tau^*$ and $\widetilde\tau$ are similarly distributed as the Fedorov-Wynn-based sensors $\tau^F$ except for goal $U_r^{(55)}$, where the $\tau^F$ exhibits a low EIG. 
Comparing the EIM-based indices $\widetilde\tau$ with the sensors $\tau^*$ that optimize the EIG lower bound, we observe in \Cref{fig:GoalV_sensors} that $\widetilde\tau$ tends to be more spread over the domain of high variability in the goal-oriented modes $V_s^*(U_r^{(i)})$. This behaviour could explain their slightly better performance over $\tau^*$ relative to the \textit{true} EIG. 

\begin{figure}[htbp]
	\centering
	\includegraphics{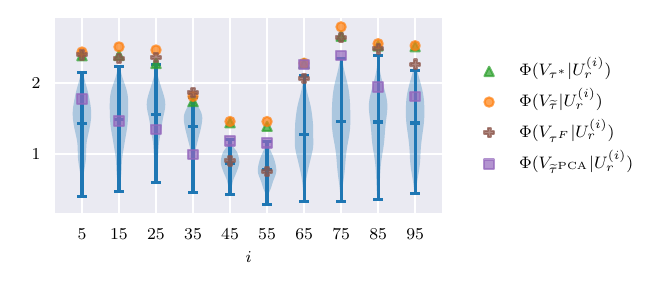}
	\caption{Goal-oriented EIG values of different goals $U_r^{(i)}$ and sensor placements. Violin plots represent $250$ random sensor placements. }
	\label{fig:EIG_burgers}
\end{figure}

\begin{figure}[htbp]
	\centering
	\begin{subfigure}[b]{\textwidth}
		\centering
		\includegraphics[width=0.47\textwidth]{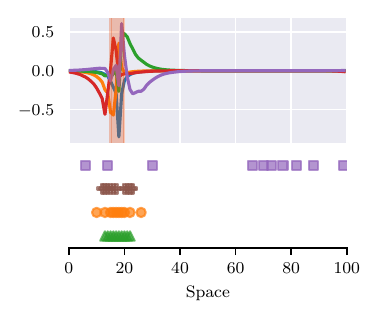}
		\includegraphics[width=0.47\textwidth]{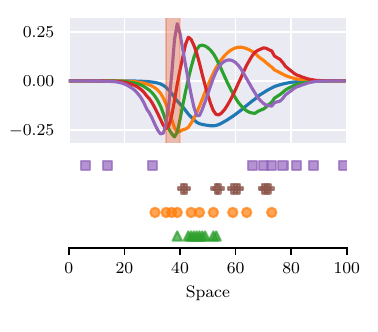}
	\end{subfigure}
	\begin{subfigure}[b]{\textwidth}
		\centering
		\includegraphics[width=0.47\textwidth]{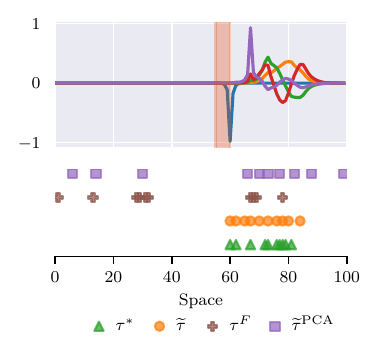}
		\includegraphics[width=0.47\textwidth]{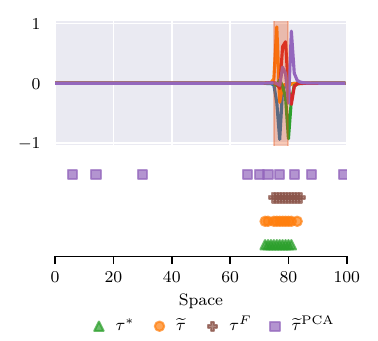}
	\end{subfigure}
	\caption{Goal-oriented modes $V_s^*(U_r)$ (denoted by lines) and sensor placements (denoted by markers at select indices below each plot) for four different parameters of interest $U_r^\top X$ marked by orange bars. }
	\label{fig:GoalV_sensors}
\end{figure}

\section{Conclusions}\label{sec:conclusion}

We derived a coupled input-output dimension reduction method that is easily computed using an alternating eigendecomposition of two diagnostic matrices. The reduced spaces are identified given only access to matrix-vector products of the model's gradient. Our method supports goal-oriented dimension reduction when some quantity of interest is prescribed in the input or the output space. We demonstrate our dimension reduction framework on Bayesian optimal experimental design and global sensitivity analysis, which can be seen as selecting the most important components of each variable instead of subspaces. In both applications, we not only bypass the expensive evaluation of the objective function for each task but also a combinatorial optimization problem by computing the relevant components from the largest diagonal entries of the diagnostic matrices.

Next, we explain a few possible directions for further developments. First, the convergence of the alternating eigendecomposition algorithm and the (adaptive) choice of the ranks $r$ and $s$ are important open questions. Second, we plan to examine quantities of interest described by some \emph{nonlinear} functions $\psi_r\colon\bbR^d\rightarrow\bbR^r$ for the input and/or $\phi_s\colon\bbR^m\rightarrow\bbR^s$ for the output. By the chain rule, the diagnostic matrices become
\begin{align*}
	H_X(\phi_s) &= \bbE\left[\nabla G(X)^\top\nabla\phi_s(G(X))^\top \nabla\phi_s(G(X)) \nabla G(X)\right],\\
	H_Y(\psi_r) &= \bbE\left[\nabla G(X)\nabla\psi_r(X)^\top \nabla\psi_r(X) \nabla G(X)^\top\right].
\end{align*}
In particular, a nonlinear output quantity of interest allows for more general goal-oriented BOED problems, while a nonlinear transformation $\psi_r$ could map the inputs to independent components (e.g. through a Knothe-Rosenblatt transform) for GSA.
The error analysis associated with such diagnostics is a natural question for future work. 
Third, our coupled dimension reduction framework can also be extended to operator learning \cite{lu2021learning,kovachki2024operator} where $G$ is considered as operator $G\colon\mathcal{U}\rightarrow\mathcal{V}$ acting on function spaces $\mathcal{U},\mathcal{V}$, typically separable Hilbert spaces. This extension is possible since the Poincar\'{e} and Cram\'{e}r-Rao inequalities are known to hold in infinite dimensions.

\section*{Acknowledgements}
The authors appreciate discussions of some of the results with Matthew Li.
The authors acknowledges support from the ANR JCJC project MODENA (ANR-21-CE46-0006-01). RB also acknowledges support from from the von K\'{a}rm\'{a}n instructorship at Caltech, the Air Force Office of Scientific Research MURI on “Machine Learning and Physics-Based Modeling
and Simulation” (award FA9550-20-1-0358) and a Department of Defense (DoD) Vannevar Bush Faculty Fellowship (award N00014-22-1-2790) held by Andrew M. Stuart.


\appendix

\section{Linking the $L^2$-error of $\widetilde G$}

Recall Proposition 1 from \cite{baptista2022JointDR}. 

\begin{proposition}\label{prop:optimal_pi_tilde}
	Set $\widetilde\pi_{X|Y} = \pi_{X_r|Y_s}\pi_{X_\perp|X_r}$, where $X_\perp = U_\perp^\top X$ is the orthogonal complement of $X_r$ with $U_\perp\in\bbR^{d\times (d-r)}$ spanning $\Ker(U_r^\top)$. Then, we have
	\begin{equation*}
		\bbE_Y\left[\KL\left(\pi_{X|Y}\|\,\widetilde\pi_{X|Y}\right)\right] \leq \bbE_Y\left[\KL\left(\pi_{X|Y}\|\,\widehat\pi_{X|Y}\right)\right] 
	\end{equation*}
	for any approximate posterior of the form $\widehat\pi_{X|Y}(x|y) = f_1(x_r,y_s)f_2(x_\perp,x_r)$ with non-negative functions $f_1,\; f_2$.
\end{proposition}

\subsection{Proof of \Cref{Lem:L2vsPosterior}}\label{app:L2vsPosterior}
First note that the normalized form of $\widetilde\pi_{X|Y}\propto \pi_{Y_s|X_r}\pi_X$ in \eqref{eq:def_pi_tilde} can be written as
	\begin{align*}
		\widetilde\pi_{X|Y} = \frac{\pi_{Y_s|X_r}\pi_X}{\int\pi_{Y_s|X_r}\pi_X\d x} = \frac{\pi_{Y_s|X_r}\pi_X}{\pi_{Y_s}} = \frac{\pi_{Y_s,X_r}\pi_{X}}{\pi_{X_r}\pi_{Y_s}} = \pi_{X_r|Y_s}\pi_{X_\perp|X_r}
	\end{align*}
	Thus, it corresponds to the definition in \Cref{prop:optimal_pi_tilde} and we can use it to obtain
	\begin{equation}\label{eq:proof_L2vsPosterior1}
		\bbE_Y\left[\KL\left(\pi_{X|Y}\|\,\widetilde\pi_{X|Y}\right)\right] \leq \bbE_Y\left[\KL\left(\pi_{X|Y}\|\,\widehat\pi_{X|Y}\right)\right] 
	\end{equation}  
	where $\widehat\pi_{X|Y}(x|y) \propto \exp(-\tfrac{1}{2\sigma^2}\|\widetilde G(x) - y\|^2)\pi_X(x)$ satisfies the definition in \Cref{prop:optimal_pi_tilde} with $f_1(x_r,y_s) = \tfrac{1}{Z} \exp(-\tfrac{1}{2\sigma^2}\|\widetilde g(x_r) - y_s\|^2)$ and $f_2(x_\perp,x_r) = \pi_X(x)$. Here, $Z\in\bbR$ represents a normalization constant independent of $x$ and $y$.  
	By Proposition 4.1 in \cite{zahm2021datafree} we have the following bound for any posterior $\pi_{X|Y}$ with Gaussian likelihood and its approximation $\widehat\pi_{X|Y}$ with an approximate forward model
	\begin{equation}\label{eq:proof_L2vsPosterior2}
		\bbE_Y\left[\KL\left(\pi_{X|Y}\|\,\widehat\pi_{X|Y}\right)\right] \leq \tfrac{1}{2\sigma^2} \bbE_Y\bigl[\|G(x) - \widetilde G(x)\|^2\bigr].
	\end{equation}
	Combining \eqref{eq:proof_L2vsPosterior1} and \eqref{eq:proof_L2vsPosterior2} finishes the proof.

\subsection{Proof of \Cref{Lem:L2vsEIG}}\label{app:L2vsEIG}
We have from the previous proof that $\widetilde\pi_{X|Y} = \pi_{X_r|Y_s} \pi_{X_\perp|X_r}$. We get for the goal-oriented EIG
\begin{align*}
	\Phi(V_s|U_r)
	&= \bbE_{Y}\left[ \KL\left( \pi_{X_r|Y_s}\|\,\pi_{X_r} \right) \right] 
	= \int \log\left(\frac{\pi_{X_r|Y_s}}{\pi_{X_r}}\right)\d\pi_{XY} \\
	&= \int \log\left(\frac{\widetilde\pi_{X|Y}}{\pi_{X_r}\pi_{X_\perp|X_r}}\right)\d\pi_{XY} \\
	&= \int \log\left(\frac{\pi_{Y|X}}{\pi_{Y}}\right)\d\pi_{XY} - \int \log\left(\frac{\pi_{X|Y}}{\widetilde\pi_{X|Y}}\right)\d\pi_{XY}\\
	&= \int \left(\int\log\left(\frac{\pi_{X|Y}}{\pi_X}\right)\d\pi_{X|Y}\right)\d\pi_Y - \int \left(\log\left(\frac{\pi_{X|Y}}{\widetilde\pi_{X|Y}}\right) \d\pi_{X|Y} \right)\d\pi_{Y}\\
	&= \bbE_Y\left[\KL\left(\pi_{X|Y}\|\,\pi_X\right)\right] - \bbE_{Y}\left[ \KL\left( \pi_{X|Y}\|\,\widetilde\pi_{X|Y} \right) \right].
\end{align*}
Applying \Cref{Lem:L2vsPosterior} on the last term concludes the proof.

\section{Proof of \Cref{thm:optimalG}}\label{app:ProofOptimalG}
Let us rewrite
\begin{align*}
	\bbE\big[\|G(X)- \widetilde G(X)\|^2\big]
	&= \bbE\Big[\big\| \big(G(X)- G^*(X)\big) + \big(G^*(X)- \widetilde G(X)\big)\big\|^2\Big] \\
	&= \bbE\big[\| G(X)- G^*(X)\|^2] +\bbE[\| G^*(X)- \widetilde G(X)\|^2\big] \\
	&\qquad + 2\bbE\Big[\big(G(X)- G^*(X)\big)^\top\big( G^*(X)-\widetilde  G(X)\big)\Big] .
\end{align*}
We show that the last term vanishes so that for any $\widetilde G$ it holds $\bbE[\|G(X)- \widetilde G(X)\|^2]\geq \bbE[\| G(X)- G^*(X)\|^2]$. This proves that $G^*$ is a minimizer of $\widetilde G\mapsto \bbE[\| G(X)- \widetilde G(X)\|^2]$. 
Let $V_\perp\in\bbR^{m\times(m-s)}$ be the orthogonal complement of $V_s$. 
Without loss of generality, we assume that $\widetilde G_\perp \in\Ker(V_s)$ so that $V_\perp V_\perp^\top\widetilde G_\perp = \widetilde G_\perp $. By construction, we have
\begin{align*}
	G(X)- G^*(X)
	&= V_s \big( V_s^\top G(X)- g^*(U_r^\top x) \big) +  V_\perp V_\perp^\top \big(G(X)-\bbE[G(X)]\big), \\
	G^*(X)-\widetilde  G(X)
	&= V_s \big( g^*(U_r^\top x) - \widetilde g(U_r^\top x) \big) +  V_\perp V_\perp^\top\big(G_\perp^*  - \widetilde G_\perp \big). 
\end{align*}
Thus, we get that 
\begin{align*}
	\bbE\Big[\big(G(X) - G^*(X)\big)^\top& \big( G^*(X)-\widetilde  G(X)\big)\Big]\\
	&= \bbE\left[ \big(V_s^\top G(X)- g^*(U_r^\top X)\big)^\top\big(g^*(U_r^\top X) - \widetilde g(U_r^\top X)\big) \right]\\
	&\qquad + \underbrace{\bbE\Big[ \big( G(X)-\bbE[G(X)] \big)^\top V_\perp V_\perp^\top \big( G_\perp^*  - \widetilde G_\perp \big)\Big]}_{=0},
\end{align*}
where the last term vanishes by taking the expectation of $G(X)-\bbE[G(X)]$.
Next, we exploit the fundamental property of the conditional expectation $\bbE[ G(X) f(U_r^\top X)] = \bbE[ \bbE[G(X)|U_r^\top X] f(U_r^\top X)] $ which holds for any function $f\colon\bbR^r\rightarrow\bbR$. Considering the definition of $g^*$ and setting $f(U_r^\top X) =V_s (g^*(U_r^\top X) - \widetilde g(U_r^\top X)) $, we obtain
\begin{equation*}
	\bbE\Big[\big(G(X)- G^*(X)\big)^\top\big( G^*(X)-\widetilde  G(X)\big)\Big]
	= \bbE\left[ \big(G(X)-  \bbE[ G(X) |\, U_r^\top X ] \big)^\top f(U_r^\top X)  \right] = 0.
\end{equation*}

To prove \eqref{eq:error_decomp}, we again use the fundamental property of the conditional expectation to write
\begin{align*}
	\bbE\big[\|G(X)& - G^*(X)\|^2\big]
	= \bbE\Big[\big\|G(X)-  V_s V_s^\top \bbE[\left. G(X) \right| U_r^\top X ] -(I_m - V_sV_s^\top) \bbE[G(X)]\big\|^2\Big] \\
	&= \bbE\Big[\big\| \big(G(X)-\bbE[G(X)]\big) -  V_s V_s^\top \big(\bbE[\left. G(X) \right| U_r^\top X ] - \bbE[G(X)]\big) \big\|^2\Big] \\
	&= \bbE\Big[\big\| G(X)-\bbE[G(X)]  \big\|^2\Big] - \bbE\Big[\big\| V_s^\top \big(\bbE[\left. G(X) \right| U_r^\top X ] - \bbE[G(X)]\big) \big\|^2\Big]  \\
	&= \Tr\left(\Cov\left(G(X)\right)\right) - \Tr\left(V_s^\top \Cov\left(\bbE\left[\left. G(X)\right| U_r^\top X \right]\right) V_s\right) ,
\end{align*}
which concludes the proof.

\section{Proof of \Cref{thm:errbound}}\label{app:ProofErrBound}

Let $V_\perp\in\bbR^{m\times(m-s)}$ be the orthogonal complement of $V_s$. We have
\begin{align}
	\bbE\big[\|G(X)& - G^*(X)\|^2\big]
	= \bbE\Big[\bigl\|G(X)-  V_s V_s^\top \bbE[\left. G(X)\right| U_r^\top X ] - V_\perp V_\perp^\top \bbE[G(X)]\bigr\|^2\Big] \nonumber\\
	&= \bbE\Big[\bigl\| V_\perp V_\perp^\top \big(G(X)-\bbE[G(X)]\big) -  V_s V_s^\top \bigl(\bbE[\left. G(X) \right| U_r^\top X ] - G(X)\bigr) \bigr\|^2\Big] \nonumber\\
	&= \bbE\Big[\big\| V_\perp V_\perp^\top\big(G(X)-\bbE[G(X)]\big) \big\|^2\Big] + \bbE\Big[\bigl\| V_s V_s^\top \bigl(\bbE[ \left. G(X) \right| U_r^\top X ] - G(X)\bigr) \bigr\|^2\Big] \label{eq:proof_errbound1}
\end{align}
Using the Poincar\'{e} inequality \eqref{eq:PI} for the component $f_i(X) = [V_\perp V_\perp^\top G(X)]_i$ yields
\begin{align}
	\bbE\Big[\bigl\| V_\perp V_\perp^\top \big(G(X)-\bbE[G(X)]\big) \bigr\|^2\Big]
	&=\sum_{i=1}^m \bbE\left[ \big( f_i(X)-\bbE[f_i(X)]\big)^2\right] \nonumber\\
	&\leq\sum_{i=1}^m  \CC(X)\, \bbE\left[\|  \nabla f_i(X)\|^2\right] \nonumber\\
	&=  \CC(X)\, \bbE\Bigl[ \big\| (I_m - V_sV_s^\top) \nabla G(X) \big\|_F^2\Bigr] \nonumber\\
	&=  \CC(X)  \Big(\bbE\left[ \| \nabla G(X)\|_F^2\right]  - \bbE\left[ \| V_s^\top \nabla G(X) \|_F^2\right] \Big) \label{eq:proof_errbound2}
\end{align}
Moreover, the subspace Poincar\'{e} inequality \eqref{eq:subspacePI} for $f_i(X) = [V_s V_s^\top G(X)]_i$ yields
\begin{align}
	\bbE\Big[\big\| V_s V_s^\top \big(\bbE[\left. G(X) \right| U_r^\top X ]& - G(X)\big) \big\|^2\Big]
	=\sum_{i=1}^m \bbE\Big[\bbE\left[\left. ( f_i(X)-\bbE[\left.f_i(X)\right| U_r^\top X])^2 \right| U_r^\top X \right]\Big] \nonumber\\
	&\leq\sum_{i=1}^m \bbE\Big[\CC(X_\perp|\,X_r)\, \bbE\big[\left.\| (I_d - U_r U_r^\top) \nabla f_i(X)\|^2 \right| U_r^\top X \big]\Big]\nonumber \\
	&\leq\sum_{i=1}^m \overline{\CC}(X)\, \bbE\left[\| (I_d-U_rU_r^\top) \nabla f_i(X)\|^2\right]\nonumber \\
	&= \overline{\CC}(X)\,  \bbE\left[ \big\| V_sV_s^\top \nabla G(X) (I_d-U_rU_r^\top) \big\|_F^2\right] \nonumber\\
	&= \overline{\CC}(X) \Bigl( \bbE\left[ \| V_s^\top \nabla G(X)  \|_F^2 \right] - \bbE\left[ \| V_s^\top \nabla G(X)  U_r \|_F^2  \right] \Bigr) \label{eq:proof_errbound3}
\end{align}
Finally, since $ \CC(X)\leq\overline{\CC}(X)$, injecting \eqref{eq:proof_errbound2} and \eqref{eq:proof_errbound3} in \eqref{eq:proof_errbound1} yields the upper bound \eqref{eq:error_bound}.
The lower bound \eqref{eq:error_lowerbound} is obtained in the same way by permuting the expectation and the norm in the above calculation. This concludes the proof.

\section{Cram\'{e}r-Rao-like Inequality}\label{app:ProofCRI}

\begin{lemma}\label{lem:CRI} 
Let $X$ be a random variable on $\bbR^d$ with Lebesgue density $\pi_X$ such that $\supp(\pi_X) = \bbR^d$. Let $f\colon\bbR^d \rightarrow \bbR$ be a differentiable function. Then it holds
	\begin{equation}\label{eq:CRI}
		c(X)\big\|\bbE\left[\nabla f(X)\right]\big\|^2 
		\leq
		\bbE\big[\big(f(X)-\bbE[f(X)]\big)^2\big] ,
	\end{equation}
	where $ c(X)= \lambda_{\max }(\bbE[\nabla \log \pi_X(X) \nabla \log \pi_X(X)^\top])^{-1} $.
\end{lemma}

\begin{proof}
	By partial integration, it holds for any differentiable function $h\colon\bbR^d\rightarrow\bbR$ 
	\begin{align}
		\bbE[\nabla h(X)] & =\int \nabla h(x) \pi_X(x) \d x \nonumber 
		 =-\int h(x) \nabla \pi_X(x) \d x \nonumber  \\
		& =-\int h(x) \nabla \log \pi_X(x) \pi_X(x) \d x 
		 =-\bbE\left[h(X) \nabla \log \pi_X(X)\right] \label{eq:proof_Cramer-Rao1}
	\end{align}
	In the second to last equality we used the chain rule $\nabla \log h(x) = \frac{\nabla h(x)}{h(x)}$.
	Let $\alpha \in \bbR^d$ be such that $\|\alpha\|=1$. Applying \eqref{eq:proof_Cramer-Rao1} to $h(x) = f(x) - \bbE[f(X)]$ and using the Cauchy-Schwarz inequality yields 
	\begin{align*}
		\left|\bbE[\nabla f(X)]^\top \alpha\right|^2 
		& = \left|\bbE[\nabla h(X)]^\top \alpha\right|^2 
		 =\bbE\left[h(X)\left(\nabla \log \pi_X(X)^\top \alpha\right)\right]^2 \\
		& \leq \bbE\left[h(X)^2\right] \bbE\left[\nabla \log \pi_X(X)^\top \alpha\right]^2 \\
		& =\Var(f(X))\, \alpha^\top \bbE\left[\nabla \log \pi_X(X) \nabla \log \pi_X(X)^\top\right] \alpha \\
		& \leq \Var(f(X))\, \lambda_{\max }\left(\bbE\left[\nabla \log \pi_X(X) \nabla \log \pi_X(X)^\top\right]\right)
	\end{align*}
	Thus, we obtain inequality \eqref{eq:CRI}.
\end{proof}

We call \eqref{eq:CRI} the Cram\'{e}r-Rao-like inequality since it poses a similar lower bound to the variance involving the inverse Fisher information matrix in $ c(X)$ \cite{rao1945AccuracyAttainable}. 
The following Lemma establishes sufficient conditions for every conditional distribution of $X_\perp$ to satisfy the Cram\'{e}r-Rao inequality, i.e., $\overline c(X)\geq 0$. 

\begin{lemma}[see \Cref{prop:control_constants}]\label{lem:control_subspaceCRI}
	Let $X$ be a random variable on $\bbR^d$ with Lebesgue density $\pi_X$ such that $\supp(\pi_X) = \bbR^d$ and there exists $\eta<\infty$ so that 
	\begin{equation}\label{eq:CRI_assumption_appendix}
		-\nabla^2\log\pi_X(x) \preceq \eta\, I_d . 
	\end{equation}
	Then there exists a positive constant $ \overline{c}(X) \geq 1/\eta$ such that 
	\begin{equation}\label{eq:subspaceCRI}
		\overline{c}(X) \left\| (I_d-U_rU_r^\top)\bbE[\nabla f(X)] \right\|^2
		\leq 
		\bbE\bigl[ \left(f(X)-\bbE[\left. f(X)\right|U_r^\top X]\right)^2\bigr] ,
	\end{equation}
	holds for any differentiable function $f: \bbR^d \rightarrow \bbR$ with $\bbE[f(X)^2]<\infty$ and $\bbE[\|\nabla f(X)\|^2]<\infty$ and for any matrix $U_r\in\bbR^{d\times r}$ with orthogonal columns.
\end{lemma}

\begin{proof}
	Let $U_\perp\in\bbR^{d\times(d-r)}$ be orthogonal to $U_r$ so that $X = U_r X_r + U_\perp X_\perp$. Further, let $\nabla_{\perp}$ denote the derivative with respect to the argument $X_{\perp}$. We write $\bbE_\perp$ and $\Var_\perp$ for the expectation and variance with respect to the conditional density $\pi_{X_\perp|X_r}$. 
	From the proof of \Cref{lem:CRI}, we have for any differentiable function $h\colon\bbR^d\rightarrow\bbR$
	\begin{equation*}
		\bbE_\perp\left[ \nabla_{\perp} h\left(X_{\perp}, x_r\right) \right]=-\bbE_\perp\left[h\left(X_{\perp}, x_r\right) \nabla_{\perp} \log \pi_{X_\perp|X_r}\left(X_{\perp} | x_r\right) \right]
	\end{equation*}
	Let $\alpha \in \bbR^d$ such that $\|\alpha\|=1$. Similar to the proof of \Cref{lem:CRI}, we set $h(x_\perp,x_r) = f(x_\perp,x_r) - \bbE_\perp[f(X_\perp,x_r)]$ to be the centered version of $f$ and obtain
	\begin{align}
		\left|\bbE_\perp\left[\nabla_{\perp} f\left(X_{\perp}, x_r\right) \right]^\top \alpha\right|^2 & = \left|\bbE_\perp\left[\nabla_{\perp} h\left(X_{\perp}, x_r\right) \right]^\top \alpha\right|^2 \nonumber\\
		& \leq  \bbE_\perp\left[ h\left(X_{\perp}, x_r\right)^2\right] \bbE_\perp\left[\left(\nabla_{\perp} \log \pi_{X_\perp|X_r}\left(X_{\perp} | x_r\right)^{\top} \alpha\right)^2 \right] \nonumber\\
		& = \Var_\perp\left(f\left(X_{\perp}, x_r\right)\right) \alpha^\top \CI_{\perp}\left(x_r\right) \alpha , \label{eq:tmp23985713}
	\end{align}
	where
	\begin{align*}
		\CI_{\perp}\left(x_r\right) & =\int \nabla_{\perp} \log \pi_{X_\perp|X_r}\left(x_{\perp}| x_r\right) \nabla_{\perp} \log \pi_{X_\perp|X_r}\left(x_{\perp} |  x_r\right)^{\top} \d \pi_{X_\perp|X_r} \\
		& =\int \nabla_{\perp} \pi_{X_\perp|X_r}\left(x_{\perp} |  x_r\right) \nabla_{\perp} \log \pi_{X_\perp|X_r}\left(x_{\perp} |  x_r\right)^{\top} \d x_{\perp} \\
		& =-\int \nabla_{\perp}^2 \log \pi_X\left(x_{\perp} |  x_r\right) \pi_{X_\perp|X_r}\left(x_{\perp} |  x_r\right)\d x_{\perp} \\
		& =-U_{\perp}^{\top} \int \nabla^2 \log \pi_X(x) \d \pi_{X_\perp|X_r}\, U_{\perp} .
	\end{align*}
	The second line uses the chain rule for $\nabla\log$ while the third uses partial integration.
	In the last line, we apply the relationship $\log \pi_{X_\perp|X_r} = \log\pi_X - \log\pi_{X_r}$ as well as the chain rule for $\nabla_{\perp}$.
	Due to the assumption $-\nabla^2 \log \pi_X(x) \preceq \eta\, I_d$, we have 
	\begin{equation*}
		\CI_{\perp}\left(x_r\right) \preceq \eta\, I_{d-r}
	\end{equation*}
	for any $U_r$.
	Thus, Equation \eqref{eq:tmp23985713} yields 
	\begin{equation*}
		\left|\bbE_\perp\left[\nabla_{\perp} f\left(X_{\perp}, x_r\right) \right]^\top \alpha\right|^2 \leq \eta\, \Var_\perp\left(f\left(X_{\perp}, x_r\right)\right)
	\end{equation*}
	so that taking the expectation over the marginal $\pi_{X_r}$ and applying Jensen's inequality on the left hand side, we get
	\begin{equation*}
		\left|\alpha^{\top} U_{\perp}^{\top} \int \nabla f(x) \d \pi_X\right|^2 \leq \eta\, \bbE\bigl[\left(f(X)-\bbE[\left.f(X) \right|  U_r^\top X]\right)^2\bigr]
	\end{equation*}
	Taking the supremum over $\|\alpha\|=1$ finally yields
	\begin{equation*}
		\frac{1}{\eta}\left\|(I_d-U_rU_r^\top)\bbE\left[ \nabla f(X) \right]\right\|^2 \leq \bbE\bigl[\left(f(X)-\bbE[\left.f(X)\right|  U_r^\top X]\right)^2\bigr]
	\end{equation*}
	for any $U_{r}$.
\end{proof}

\begin{remark}\label{rem:interpret_CRI_assumption}
	Assumption \eqref{eq:CRI_assumption_appendix} is equivalent to saying that the function $h(x) := \frac{\eta}{2}\|x\|^2 + \log\pi_X(x)$ is convex. This implies that $h(x)\geq h(y)+\nabla h(y)^\top(x-y)$ for all $x,y\in\bbR^d$ and, by letting $y=0$, we deduce that
	\begin{equation*}
		\pi_X(x) \geq C \exp\left( - \frac{\eta}{2}\|x- \eta^{-1}\nabla\log\pi_X(0) \|^2 \right),
	\end{equation*}
	holds for all $x\in\bbR^d$, where $C>0$ is a constant independent of $x$. This means that if \eqref{eq:CRI_assumption} holds, then the density $\pi_X$ is necessarily above some rescaled Gaussian density.
\end{remark}

\section{Proof of \Cref{prop:GradvsGSA}}\label{app:ProofGradvsGSA}

On one side, the law of total variance enables us to write
\begin{align*}
	S^{\cl}(U_{\tau}|V_s) 
	&= 1-\frac{\bbE\bigl[\| V_s^\top G(X)-\bbE[\left.V_s^\top G(X)\right| U_{\tau}^\top X] \|^2\bigr]}{\Tr\left(\Cov\left(V_s^\top G(X)\right)\right)}  \\
	S^{\tot}(U_{\tau}|V_s) 
	&= \frac{\bbE\bigl[\| V_s^\top G(X)-\bbE[\left.V_s^\top G(X)\right| U_{-\tau}^\top X] \|^2\bigr]}{\Tr\left(\Cov\left(V_s^\top G(X)\right)\right)} .
\end{align*}
On the other side, applying \Cref{thm:errbound} with $V_s^\top G(X)$ instead of $G(X)$ yields
\begin{align*}
	\bbE\big[\| V_s^\top G(X) &- \bbE[\left. V_s^\top G(X)\right| U_{\tau}^\top X]\|^2\big]\\ 
	&\leq \overline{\CC}(X) \Big(  \bbE\left[\| V_s^\top\nabla G(X) \|_F^2\right]-  \bbE\bigl[\| V_s^\top \nabla G(X) U_{\tau} \|_F^2\bigr] \Big)  \\
	&= \overline{\CC}(X)   \Tr( U_{-\tau}^\top H_X(V_s) U_{-\tau})   \\
	\bbE\big[\| V_s^\top G(X)&- \bbE[\left.V_s^\top G(X)\right| U_{\tau}^\top X]\|^2\big] \\
	&\geq \overline{c}(X) \Big( \big\|V_s^\top\bbE\left[ \nabla G(X) \right]\big\|_F^2 - \left\| V_s^\top \bbE\bigl[\nabla G(X) \bigr] U_{\tau} \right\|_F^2 \Big) \\
	&= \overline{c}(X)  \left\|V_s^\top\bbE\left[ \nabla G(X) \right] U_{-\tau} \right\|_F^2 
\end{align*}
By combining these relations we directly obtain \eqref{eq:bound_closedSobol} and, by replacing $U_{\tau}$ with $U_{-\tau}$ in the above inequalities, we obtain \eqref{eq:bound_totalSobol}. This concludes the proof.


\bibliographystyle{siamplain}
\bibliography{references.bib}

\end{document}